\documentclass{article}

\PassOptionsToPackage{square, numbers}{natbib}
\usepackage[preprint]{neurips_2026}

\usepackage[utf8]{inputenc} 
\usepackage[T1]{fontenc}    
\usepackage{hyperref}       
\usepackage{url}            
\usepackage{booktabs}       
\usepackage{amsfonts}       
\usepackage{nicefrac}       
\usepackage{microtype}      
\usepackage{xcolor}         

\usepackage{amsmath, amssymb, amsthm, mathtools}
\usepackage{bm}
\usepackage{mathrsfs}
\usepackage{physics}
\usepackage{cancel}
\usepackage{hyperref}
\usepackage{url}
\usepackage{graphicx}
\usepackage{algorithm}
\usepackage{algorithmic}
\usepackage{xcolor}
\usepackage{listings}
\hypersetup{
    colorlinks=true,
    linkcolor=blue,
    filecolor=blue,
    urlcolor=blue,
    pdftitle={Full Dynamical Stability Analysis of SGD},
    pdfpagemode=FullScreen,
}

\usepackage{booktabs}
\usepackage{enumitem}

\theoremstyle{plain}
\newtheorem{theorem}{Theorem}[section]
\newtheorem{proposition}[theorem]{Proposition}
\newtheorem{lemma}[theorem]{Lemma}

\theoremstyle{definition}
\newtheorem{definition}[theorem]{Definition}

\theoremstyle{remark}
\newtheorem{remark}[theorem]{Remark}

\definecolor{softpurple}{RGB}{128, 70, 170}

\usepackage{amsmath}
\usepackage{amssymb}
\usepackage{mathtools}
\usepackage{amsthm}
\usepackage{enumitem}

\usepackage{microtype}
\usepackage{graphicx}
\usepackage{subcaption}
\usepackage{booktabs} 
\usepackage{hyperref}
\usepackage{url}
\usepackage{amsthm}
\usepackage{amssymb}
\usepackage{graphicx}
\usepackage{algorithm}
\usepackage{algorithmic}
\usepackage{xcolor}
\usepackage{listings}
\makeatletter
\renewcommand\paragraph{\@startsection{paragraph}{4}{\z@}%
  {0ex \@plus 0.05ex \@minus 0.05ex}%
  {-2em}%
  {\normalfont\normalsize\bfseries}}
\makeatother
\usepackage[font=small]{caption}

\title{Sequential Subspace Noise Injection \\ Prevents Accuracy Collapse in Certified Unlearning}

\author{
Polina Dolgova$^{1,2}$ \quad
Sebastian U. Stich$^{1}$\\
$^1$CISPA Helmholtz Center for Information Security \quad
$^2$Universität des Saarlandes\\
Saarbrücken, Germany\\
\texttt{\{polina.dolgova, stich\}@cispa.de}
}

\begin{document}

\maketitle

\begin{abstract}
    Certified unlearning via differential privacy provides formal guarantees but remains impractical due to the severe accuracy collapse observed in current noisy fine-tuning (NFT) approaches. We propose Block-wise Noisy Fine-Tuning, which utilizes sequential subspace noise injection to partition the parameter space into orthogonal blocks and update only one block per step. This modification redistributes the noise budget over time; by restricting perturbations to lower-dimensional subspaces, we significantly reduce per-step distortion and prevent noise from overwhelming the training signal.
    We extend the privacy analysis to our block-wise schedule, proving that the same $(\varepsilon,\delta)$ privacy budget is retained.
    Empirical results on image classification benchmarks demonstrate that our approach prevents accuracy collapse and remains robust to membership inference attacks, bridging the gap between rigorous guarantees and practical utility.
\end{abstract}

\section{Introduction} 

Machine unlearning aims to remove the influence of specified training data from a trained model, motivated by legal requirements such as the GDPR~\citep{european_commission_regulation_2016} and practical needs such as removing sensitive or malicious data~\citep{nguyen2024surveymachineunlearning}. Retraining on the retained data is the gold standard but is often computationally prohibitive. Existing methods are commonly grouped into exact, certified, and empirical approaches~\citep{pmlr-v119-guo20c,pmlr-v132-neel21a,fan2024salun,jia2023model}: empirical methods can preserve utility but lack formal removal guarantees, while certified methods provide guarantees at the cost of greater conservatism.

This paper focuses on the certified setting, where the goal is to obtain a formal distributional guarantee rather than only strong empirical forgetting scores. Certified unlearning is typically more conservative than empirical unlearning, and its practical scalability remains limited in very large models or under tight privacy budgets. Among certified approaches, differential-privacy-inspired methods are prominent \citep{10.1561/0400000042,pmlr-v108-balle20a,liu2023certified,allouah2025the}. Recent noisy fine-tuning (NFT) with gradient clipping \citep{koloskova2025certified} provides an $(\varepsilon,\delta)$ certificate for arbitrary models and losses, but its Gaussian noise can severely degrade utility. On \textsc{CIFAR-10} with ViT-Tiny, for example, NFT drops from $88\%$ to below $20\%$ test accuracy during unlearning and does not recover after subsequent fine-tuning (Figure~\ref{fig:cifar18_drop}).

\begin{figure}
    \centering
    \includegraphics[width=0.7\linewidth]{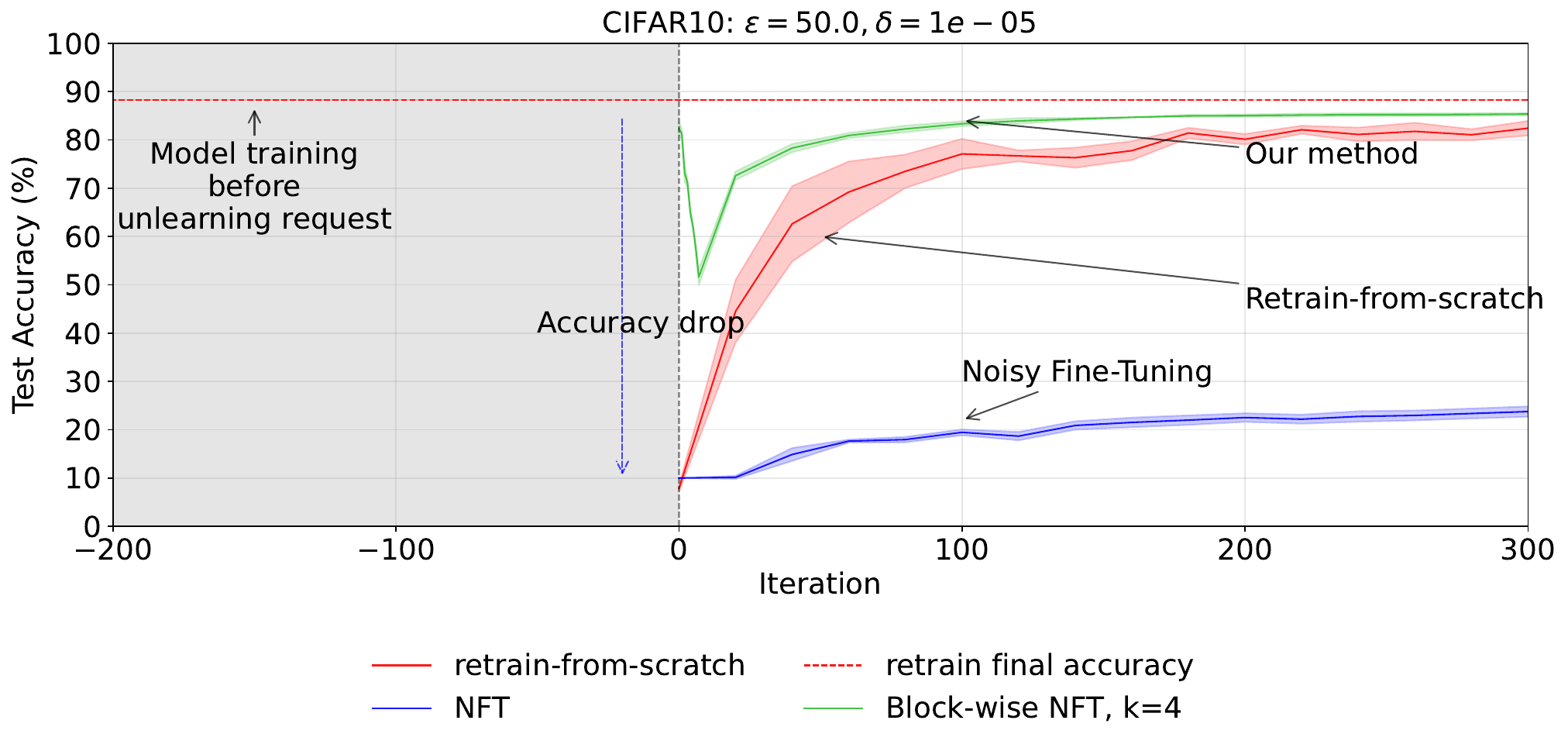}
    \caption{\textbf{Severe accuracy drop under noisy fine-tuning.} On \textsc{CIFAR-10} with ViT-Tiny, standard noisy fine-tuning (NFT, \citealp{koloskova2025certified}) test accuracy drops sharply from $88\%$ to below $20\%$ once the unlearning begins, and does not recover even after 300 subsequent fine-tuning steps.}
    \label{fig:cifar18_drop}
\end{figure}

To address this limitation, we propose \emph{Block-wise Noisy Fine-Tuning} (Block-wise NFT), a variant of NFT that improves utility while preserving certified guarantees. We partition the parameter space into mutually orthogonal subspaces and apply the noisy updates sequentially, injecting noise only in the active subspace while freezing the rest. The scalar noise variance required by the certificate is not reduced; the gain comes from perturbing only a lower-dimensional block at each step rather than the full parameter vector.
In our experiments, we use random orthogonal blocks, while other structured partitions (for example, layer-wise groupings) are possible.

We also reduce a second source of pessimism: worst-case model clipping. Standard NFT protects against arbitrary starting points, whereas certified unlearning compares the full-data model to the corresponding retained-data reference. We formalize this comparison by a high-probability proximity parameter $\Delta(\rho)$ under a fixed coupling of the two training runs. Replacing the worst-case radius by $\Delta(\rho)$ yields a certificate tailored to the relevant retraining comparison. Thus, our method addresses two NFT bottlenecks: full-space noise injection and worst-case initialization clipping.

The central challenge in certified unlearning is to design methods that \emph{simultaneously} preserve rigorous differential-privacy-based guarantees and avoid the severe accuracy degradation caused by heavy clipping and noise in current DP-inspired approaches.
Our work makes the following contributions:
\begin{itemize}[topsep=0pt,itemsep=0pt,leftmargin=12pt]
  \item \textbf{Method.} We introduce Block-wise NFT, which applies certified noisy updates sequentially across orthogonal parameter blocks, reducing per-step distortion while retaining full-parameter coverage over the schedule.
  
  \item \textbf{Certification.}
   We extend the NFT certification analysis to the
  block-wise schedule and prove that the same $(\varepsilon,\delta)$ unlearning
  guarantee is preserved. We further show that worst-case model clipping can be replaced by the calibration parameter $\Delta(\rho)$.
  
   \item \textbf{Experiments.} On \textsc{MNIST} and \textsc{CIFAR-10}, Block-wise NFT substantially reduces the accuracy collapse of standard NFT under the same $\Delta(\rho)$ calibration for both random and class-wise deletions, while remaining consistent with standard forgetting audits. We additionally report empirical unlearning
   baselines as auxiliary reference points.
    
  \item \textbf{Formal Bottleneck.} We show why worst-case, training-agnostic certification can create a utility-efficiency tension: under initialization-agnostic indistinguishability, maintaining high utility in substantially fewer steps than retraining would imply a faster retraining procedure. This explains why DP-based certified unlearning analyses can be overly conservative and motivates our proximity-based formulation.
  
\end{itemize}
In summary, our approach preserves formal certified unlearning guarantees while substantially mitigating the utility loss that has so far limited differentially private methods in practice.

\section{Preliminaries and Problem Statement} 

In this section, we establish notation and recall the standard definitions used throughout the paper.

\textbf{Setup.}
A (possibly randomized) learning algorithm $\mathcal{A}$ maps a dataset $\mathcal{D}$ to model parameters $\hat{\mathbf{x}}\in\mathbb{R}^d$, i.e., $\hat{\mathbf{x}}=\mathcal{A}(\mathcal{D})$. A deletion request specifies a subset $\mathcal{D}_f\subseteq\mathcal{D}$ to be removed; the retained data are $\mathcal{D}_r:=\mathcal{D}\setminus\mathcal{D}_f$. 
An unlearning mechanism $\mathcal{U}$ takes $(\hat{\mathbf{x}},\mathcal{D},\mathcal{D}_f)$ and, using randomness, outputs updated parameters
$\tilde{\mathbf{x}}=\mathcal{U}(\hat{\mathbf{x}},\mathcal{D},\mathcal{D}_f)$.

We adopt the definition from \citep{koloskova2025certified}, where the notion of certified approximate unlearning is introduced, building on an analogy with differential privacy.

\begin{definition}[($\varepsilon, \delta$)-unlearning \citep{koloskova2025certified}]~\label{def:unlearn}
Let $\varepsilon \geq 0$, $\delta \in [0,1]$.
We say that $\mathcal{U}$ is an $(\varepsilon, \delta)$-unlearning algorithm for $\mathcal{A}$ if there exists a certifying algorithm $\bar{\mathcal{A}}$ such that, for any forget dataset $\mathcal D_f \subset \mathcal D$ and any observation $O \subset \mathbb{R}^d$,
\begin{equation*}
\begin{aligned}
\Pr[\mathcal{U}(\mathcal{A}(\mathcal{D}), \mathcal{D}, \mathcal{D}_f) \in O] \leq e^{\varepsilon} \Pr[\bar{\mathcal{A}}(\mathcal{D} \setminus \mathcal{D}_f) \in O] + \delta,\\
\Pr[\bar{\mathcal{A}}(\mathcal{D} \setminus \mathcal{D}_f) \in O] \leq e^{\varepsilon} \Pr[\mathcal{U}(\mathcal{A}(\mathcal{D}), \mathcal{D}, \mathcal{D}_f) \in O] + \delta.
\end{aligned}
\end{equation*}
\end{definition}

Thus, throughout the paper, ``certified'' refers to closeness to the retained-data reference distribution, not to passing a particular empirical audit. Following \cite{koloskova2025certified}, we take the certifying algorithm to be $\bar{\mathcal{A}}(\mathcal{D} \setminus \mathcal{D}_f) = \mathcal{U}(\mathcal{A}(\mathcal{D}_r), \mathcal{D}_r, \varnothing)$, applying the same unlearning procedure to a retrained model. Thus, the guarantee compares the unlearned model to this retained-data reference distribution. It certifies indistinguishability, but not high utility.

We build upon the noisy fine-tuning method introduced by \citet{koloskova2025certified}, which is inspired by the standard DP-SGD algorithm \citep{abadi2016deep} and applied only to the retained data $D_r$. The method combines gradient clipping with Gaussian noise injection and is defined as follows:

\begin{definition}[Noisy fine-tuning \citep{koloskova2025certified}]~\label{def:noisyft}
\begin{subequations}\label{eq:nft}
\begin{align}
\mathbf{x}_0 &= \Pi_{C_0}(\hat{\mathbf{x}}), \quad
\mathbf{x}_{t+1} = \mathbf{x}_t - \gamma\big(\Pi_{C_1}(g_t)+\lambda \mathbf{x}_t\big)
                   + \boldsymbol{\xi}_{t+1}. 
\end{align}
\end{subequations}

where $\mathbf{x}_t$ are the parameters at iteration $t$, $g_t$ is the gradient at step $t$ (computed on $\mathcal{D}_r$), $\gamma > 0$ is the learning rate, $\lambda \geq 0$ is the weight decay parameter, $\boldsymbol{\xi}_{t+1} \sim \mathcal{N}(0, \sigma^2 I_d)$ is Gaussian noise, and  $\Pi_{C_0}, \Pi_{C_1}$ are clipping operators with radii $C_0, C_1 > 0$, defined as $\Pi_C(\mathbf{v}) := \mathbf{v} \cdot \min\bigl\{ \tfrac{C}{\|\mathbf{v}\|}, 1 \bigr\}.$
\end{definition}

For comparison, we also introduce notation for retrained models. Let $\hat{\textbf{x}}'=\mathcal{A}(\mathcal{D}_r)$ denote the model parameters obtained by training on the retained dataset $\mathcal{D}_r = \mathcal{D} \setminus \mathcal{D}_f$ with the same algorithm and a suitably \emph{fixed} coupling of randomness 
(e.g., matched random seeds or noise schedules). Accordingly, let $\textbf{x}'_t$ denote the iterates produced by applying the updates from Definition~\ref{def:noisyft} to the model $\hat{\textbf{x}}'$.

\section{Algorithm Motivation}

In this section, we identify two reasons why accuracy may degrade under certified unlearning via noisy fine-tuning (NFT), and propose two corresponding remedies. We first motivate \emph{block-wise} noise injection to reduce the destructive effect of isotropic noise, and then motivate using a \emph{proximity} parameter $\Delta(\rho)$ in place of worst-case model clipping.

\subsection{Necessity of block-wise noise scheduling}

A central challenge in certified unlearning based on noisy fine-tuning is the severe degradation of model accuracy observed during \textit{the unlearning phase}. Empirically, test accuracy often drops sharply once unlearning begins and cannot be fully recovered by subsequent fine-tuning. The reason is that the injected noise is large enough to dominate the gradient signal across all parameters.

We now formalize this limitation with a lower bound on the per-step noise level.
The bound is expressed in terms of $(q,\varepsilon^{\text{r\'enyi}})$-R\'enyi
Differential Privacy (RDP)~\citep{Mironov_2017}, where
$\varepsilon^{\text{r\'enyi}}$ denotes the privacy loss at order $q>1$. As
standard, an RDP guarantee can be converted into an
$(\varepsilon,\delta)$-guarantee (Definition~\ref{def:unlearn}) via
\[
  \varepsilon
  = \varepsilon^{\text{r\'enyi}} + \tfrac{\log(1/\delta)}{q-1}.
\]

\begin{theorem}[Per-step noise lower bound]\label{thm:noise-lb}
Let $\gamma>0$ be the learning rate and $\lambda \ge 0$ the weight decay
parameter, with $\gamma\lambda < 1$. Consider Noisy Fine-Tuning with gradient
clipping radii $C_0,C_1>0$ and Gaussian perturbations, certified via Rényi DP.
Then, within this shifted-Rényi certification analysis, any noise scale $\sigma$ that enables $(\varepsilon,\delta)$-unlearning
must satisfy
\begin{equation*}
        \sigma^2 \!
        \begin{cases}
            \ge\; \gamma(2 - \gamma \lambda)  \tfrac{2 q} {\varepsilon^{\text{r\'enyi}}} \Bigl(2 - \tfrac{\lambda C_0}{C_1}\Bigr) C_0C_1, & \text{if } \tfrac{\lambda C_0}{C_1} \in (0, 1), \\[1em]
            \;\;>\; \gamma(2 - \gamma \lambda) \tfrac{2 q} {\varepsilon^{\text{r\'enyi}}}  \tfrac{C^2_1}{\lambda}, & \text{if } \tfrac{\lambda C_0}{C_1} \in [1, \infty).
        \end{cases}
\end{equation*}
This inequality holds for \emph{any} number of unlearning steps $T$. The full dependence of the minimal step count $T(\sigma^2)$ on the noise level is derived in Appendix~\ref{app:lower-bound}. In particular, in the first regime $\tfrac{\lambda C_0}{C_1}\in(0,1)$ and for the
minimal feasible noise $\sigma^2_{\min}$, the required number of steps is
\[
  T(\sigma^2_{\min})
  \;=\;
  \frac{\log\!\bigl(1 - \tfrac{\lambda C_0}{C_1}\bigr)}{\log(1-\gamma\lambda)}.
\]
\end{theorem}

\textbf{Interpretation.}
Theorem~\ref{thm:noise-lb} is a \emph{necessary condition within the proof
framework of \citep{koloskova2025certified}}. If $\sigma^2$ falls below the
stated threshold, the divergence bound (Theorem~A.9 in their work) cannot be
satisfied, and the mechanism cannot be certified as
$(\varepsilon,\delta)$-unlearning by this analysis. This does not rule out that
other algorithms or analyses might achieve valid guarantees with smaller noise:\
the result is proof-technique limited, not an information-theoretic
impossibility.

This observation reveals a practical bottleneck of NFT in over-parameterized networks: within this framework, a non-negligible amount of Gaussian noise must be injected at \emph{every} unlearning step. The key remaining choice is therefore how to distribute the noise across coordinates during optimization.

\textbf{From global noise to blocks.}
To avoid perturbing all parameters simultaneously, we use a block-wise (subspace-wise) schedule: at each step, we inject noise and apply the clipped update only within one orthogonal block, keeping the others fixed. Redistributing the same certified noise budget across blocks and time reduces per-step distortion and makes subsequent fine-tuning more stable.

\textbf{Intuition.}
Even at the minimal feasible noise level $\sigma_{\min}$, every coordinate receives Gaussian noise at each step, so the noise vector typically has $\ell_2$-norm about $\sigma_{\min}\sqrt{d}$. For networks with millions of parameters, 
this perturbation causes the additive noise to dominate the \emph{clipped} update across many coordinates, explaining the sharp accuracy degradation observed in practice.
The two regimes
$\tfrac{\lambda C_0}{C_1}<1$ versus $\tfrac{\lambda C_0}{C_1}\ge1$
reflect whether gradient clipping or weight decay dominates the dynamics.

Taken together, the theorem highlights why NFT struggles in over-parameterized models: certification forces the injection of noise into \emph{all} coordinates at each step. This motivates our adaptation based on \emph{sequential subspace injection}, where the same noise budget is redistributed across orthogonal subspaces instead of being applied globally.

\subsection{Retrained model localization}

\label{sec:C0}

Independent of how noise is injected, certified unlearning methods whose guarantees are formulated in a \emph{training-agnostic, worst-case} manner must account for arbitrary model initializations, which can make it impossible to guarantee \emph{both} good performance and faster-than-retraining unlearning.

We illustrate this limitation through a simple thought experiment in the context of NFT, which motivates our later replacement of the model clipping radius~$C_0$ with a tighter closeness parameter~$\Delta(\rho)$.

\textbf{Setup.}
Let $T_{\mathrm{retrain}}$ denote the minimal number of training steps required to retrain a model from scratch on the retain set $D_r$ so as to achieve the same performance guarantee as that certified for NFT after $T$ unlearning steps. By construction, any procedure that attains this guarantee in fewer than $T_{\mathrm{retrain}}$ steps would constitute a faster retraining algorithm, contradicting the definition of $T_{\mathrm{retrain}}$.

In the standard NFT analysis, the number of unlearning steps $T$ is fixed in advance as a function of the privacy budget $(\varepsilon,\delta)$, the learning
rate, and the clipping parameters. Crucially, $T$ does \emph{not} depend on the initialization of the model parameters. Moreover, the certified guarantee is
formulated in the worst case:\ it requires NFT to produce indistinguishable outcomes for \emph{any} two initializations.

\textbf{Thought experiment.}
Suppose that NFT, when initialized at the fully trained model $\hat{\textbf{x}}$, reaches performance at least $\alpha$ with probability $p$ after $T$ steps, where $T \ll T_{\mathrm{retrain}}$. Since the guarantee is initialization-agnostic and holds for any model, we may instead initialize NFT from a randomly initialized model $\mathbf{x}_{\mathrm{init}}$. Moreover, by the $(\varepsilon,\delta)$ guarantee, the probabilities of unfavorable performance events under these two initializations are comparable:

\begin{proposition}~\label{prop:thought-exp}
Fix a performance threshold $\alpha$. Let $\mathrm{Perf}:\mathbb{R}^d\to\mathbb{R}$ denote the performance of a model (as a function of its parameters), and define the unfavorable event
\[
\mathcal{E}_{\mathrm{fail}} := \{\, \mathbf{x}\in\mathbb{R}^d : \mathrm{Perf}(\mathbf{x}) < \alpha \,\}.
\]
If $\mathcal{U}$ satisfies the worst-case, initialization-agnostic certification~\eqref{eq:init-agnostic}, then
\[
\Pr\!\big[\mathcal{U}(\mathbf{x}_{\mathrm{init}})\in \mathcal{E}_{\mathrm{fail}}\big]
\;\le\;
e^{\varepsilon}\Pr\!\big[\mathcal{U}(\hat{\mathbf{x}})\in \mathcal{E}_{\mathrm{fail}}\big] + 2\delta.
\]
\end{proposition}

We provide the proof in App.~\ref{app:dp-perf}. Thus, after the same $T$ steps, this procedure must produce a model whose performance is $(\varepsilon, \delta)$-close to that of NFT started from $\hat{\textbf{x}}$. This, in turn, suggests that retraining from scratch can reach accuracy close to $\alpha$ in just $T$ steps. This directly contradicts the definition of $T_{\mathrm{retrain}}$ as the minimal number of steps required for retraining to the same certified performance level. Hence, NFT cannot at the
same time (i) guarantee good performance and (ii) be strictly faster than
retraining under the worst-case clipping analysis. The intuition behind the full thought experiment is illustrated in
Fig.~\ref{fig:illustration}.

\begin{figure}[t]
    \centering
    \includegraphics[width=0.6\linewidth]{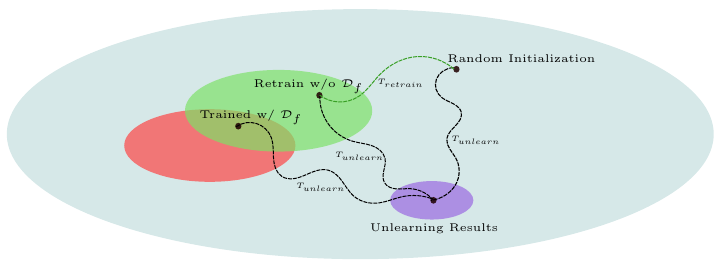}
    \caption{\textbf{Illustration of the initialization-agnostic bottleneck for Noisy Fine-Tuning.} 
    The illustration shows fully-trained model $\textbf{x}_0$, retrained model $\textbf{x}'_0$ and $\textbf{x}_{\rm init}$ after model clipping. Unlearning trajectories require $T < T_{\rm retrain}$ steps. However, we need at least $T_{\rm retrain}$ steps of unlearning on the $\textbf{x}_{\rm init}$ to reach good quality (the \textcolor{green!60!black}{green region}). Therefore, results of unlearning ($T$ steps from $\textbf{x}_{\rm init}$) cannot obtain a good quality.}
    \label{fig:illustration}
\end{figure}

\begin{remark}
Although we present this argument for NFT, the same limitation applies to any training-agnostic unlearning algorithm whose number of update steps $T$ does not depend on the model’s initialization. The argument further relies on the retraining time required to match slightly weakened performance guarantees being comparable to that for the original guarantee; we formalize this general statement and discuss its limitations in Appendix~\ref{app:dp-perf-2}.
\end{remark}

\textbf{Implication.}
The contradiction highlights that the model clipping-based guarantee is \emph{too
strong}:\ it enforces indistinguishability across all possible initializations,
even completely random ones, which is not required by the definition of
certified unlearning. In practice, the distance between the fully trained model
$\hat{\textbf{x}}$ and the retrained model $\hat{\textbf{x}}'$ is much smaller than the clipping
diameter $2C_0$. We therefore replace the initial clipping with the clipping radius $C_0$ by a high-probability bound
$\Delta(\rho)/2$ on this distance, yielding significantly tighter and more
practical guarantees in the subsequent analysis.

\begin{definition}[High-Probability Initial Discrepancy]
\label{def:initial-discrepancy}
For any failure probability $\rho \in (0,1]$, 
we define the \emph{initial discrepancy} $\Delta(\rho)$ between the fully trained model
$\hat{\textbf{x}}$ and the retrained model $\hat{\textbf{x}}'$ as the smallest value satisfying
\begin{equation}
\Pr\bigl[\|\hat{\textbf{x}}-\hat{\textbf{x}}'\|\le \Delta(\rho)\bigr]\ge 1-\rho.
\end{equation}
\end{definition}

\begin{remark}
Discrepancy $\Delta(\rho)$ for $\rho = 1$ is connected to the sensitivity \citep{10.1561/0400000042} of the model’s output. However, we are only considering a particular joint distribution $(X, X’)$ while sensitivity takes supremum over all possible $X$ and $X'$.
\end{remark}

\begin{remark}
\label{rem:coupling}
Different couplings may yield different values of $\Delta(\rho)$; 
we take the coupling to be part of the definition of the certifying 
procedure and hold it fixed throughout the analysis.
\end{remark}
\textbf{Discussion.}
Removing model clipping and calibrating noise using $\Delta(\rho)$ instead of the worst-case bound $2C_{0}$ yields smaller noise and sharper guarantees whenever the two trained starting points are close. Conceptually, this replaces a uniform guarantee over all initialization pairs in the $C_{0}$-ball with a guarantee for the specific coupled distribution of initializations induced by full-data training and retained-data training, which is exactly what the certified unlearning definition requires.

In this work, we treat $\Delta(\rho)$ as a \textit{conditioning parameter} and interpret both the analysis and the experiments for a given instantiated value of $\Delta(\rho)$. The unlearning procedure remains certified under the stated proximity condition. Empirically, we instantiate $\Delta(\rho)$ using practical estimation procedures (rather than assuming an oracle bound), and show that under this calibration our method remains competitive with existing empirical (uncertified) unlearning approaches.

\section{Algorithm}
\vspace*{-0.5ex}

\subsection{Block-wise noisy fine-tuning}

In our approach, we fix the number of subspaces (blocks) $k \in \mathbb{N}$ and partition the weight space $\mathbb{R}^d$ into $k$ mutually orthogonal components. To construct this partition, choose integers $r_1,\dots,r_k\ge 0$ with $\sum_{i=1}^k r_i=d$ and build matrices $A_i\in\mathbb{R}^{d\times r_i}$ whose columns are orthonormal. Define
\begin{equation}\label{eq:matr-a}
A \;:=\; [\,A_1\;\cdots\;A_k\,] \in \mathbb{R}^{d\times d},
\qquad\text{so that } A^\top A = I_d .
\end{equation}
Then the subspaces are $V_i := \mathrm{span}(A_i)$; they are mutually orthogonal and
$\mathbb{R}^d = \bigoplus_{i=1}^k V_i$.

Note that this construction is fully general: the subspaces $V_i$ may be chosen \emph{arbitrarily} subject to orthogonality. Moreover, since $A=[A_1\cdots A_k]$ is orthogonal, every $W\in\mathbb{R}^d$ decomposes uniquely as
$W=\sum_{i=1}^k A_iB_i$, where $B_i=A_i^\top W$; see Appendix~\ref{app:theory-minor}. Algorithm~\ref{alg:block-noisy-finetuning} then applies noisy fine-tuning sequentially to these blocks and finally performs standard fine-tuning on the full model.

\begin{algorithm}[tb]
\caption{Block-wise noisy fine-tuning for unlearning}
\label{alg:block-noisy-finetuning}
\begin{algorithmic}[1]

\REQUIRE model $\hat{\textbf{x}}$, parameters $\gamma, \lambda, \Delta(\rho), C_1$, number of blocks $k$ and privacy budget parameters $(\varepsilon, \delta)$.

\STATE Define projection matrices $A_1, \dots, A_k$.
\STATE Decompose the model weights $\hat{\textbf{x}}$ as $\hat{\textbf{x}} = \sum_{i=1}^k A_i B_i$ (the $A_i$ are fixed and not trained).
\FOR{$i = 1, \dots, k$}
    \STATE Freeze all parameters except $B_i$.
    \STATE Calculate the noise variance $\sigma^2$ and the number of steps $T$ using the formula from Theorem~\ref{thm:noise-lb} with $C_0$ replaced by $\Delta(\rho)/2$.
    \STATE Apply noisy fine-tuning to $B_i$ without the initial model-clipping step $C_0$.
\ENDFOR
\STATE Run several standard fine-tuning steps with all model parameters unfrozen.
\end{algorithmic}
\end{algorithm}

\begin{remark}
At each sequential step, noise is injected only into a single block of dimension $r_i$ (rather than all $d$ parameters). For equal-size blocks $r_i=d/k$, each step perturbs only a $1/k$ fraction of coordinates. Hence, while the noise variance $\sigma^2$ is unchanged, the per-step perturbation is smaller, helping preserve utility.
\end{remark}

\subsection{Theoretical guarantees}

We also prove that our algorithm preserves theoretical unlearning guarantees. The same statements hold in the unconditional clipping-based formulation (radius $C_0$).

\begin{theorem}\label{thm:block-guarantee}
For the decomposition $W = \sum_{i=1}^{k} A_i B_i$ and noisy fine-tuning algorithm with privacy budget parameters $(\varepsilon_i, \delta)$ there is an $(\varepsilon, \delta)$ unlearning guarantee, where~
\begin{equation}
            \varepsilon = \sum_{i=1}^{k} \varepsilon^{\text{r\'enyi}}_i + \frac{\log(1/\delta)}{q - 1}.
\end{equation}
\end{theorem}

\subsection{Proof sketch of Theorem~\ref{thm:block-guarantee}}

In this subsection, we provide a brief proof sketch and clarify the role of the main technical components; full details are given in Appendix~\ref{app:certification-blocks}. We extend the proof of \cite{koloskova2025certified}, which relies on a sequence of privacy amplification inequalities with shifted R\'enyi divergence~\citep{pmlr-v108-balle20a} as the key tool. The block decomposition induces a \emph{multi-dimensional} notion of shift, and our proof is organized around adapting the shifted-divergence framework to this setting.

Specifically, we generalize Wasserstein distance and shifted R\'enyi divergence to a vector-valued (block-wise) shift. In this setting, we adapt two key steps: the Shift Reduction Lemma~\citep{Feldman_2018} (Lemma~\ref{lem:shift-reduction}, proved in the main text) and a decomposed version of Lemma~A.7 from \cite{koloskova2025certified} (proved in Appendix~\ref{app:certification-blocks}). Together, these decomposed tools allow us to proceed block by block in the R\'enyi divergence inequalities: each block update contributes an incremental privacy loss $\varepsilon_i^{\text{r\'enyi}}$ that depends only on the parameters used for this block. Summing these per-block increments yields the overall $(\sum_{i=1}^k \varepsilon_i^{\text{r\'enyi}}, q)$-RDP guarantee and, consequently, Theorem~\ref{thm:block-guarantee}. 

\begin{definition}[Decomposition gap]\label{def:decomp-gap}
    Let $A_i$ for $i = 1, \ldots, k$ be a fixed set of matrices as defined in~(\ref{eq:matr-a}). Let
    \[
      \textstyle  W = \sum_{i=1}^k A_i B_i,
        \qquad 
      \textstyle  W' = \sum_{i=1}^k A_i B_i'.
    \]
    We define the \emph{decomposition gap} between $W$ and $W'$ as
    \[
        \mathcal{G}(W,W') := (z_1, \ldots, z_k), 
        \qquad 
        z_i := \|B_i - B_i'\|.
    \]
    For two such vectors $z^{(1)}$ and $z^{(2)}$, we write 
    $z^{(1)} \preceq z^{(2)}$ if the inequality holds coordinate-wise.
    \end{definition}

This leads to the definitions of the $\infty$-Wasserstein distance and the shifted R\'enyi divergence.
    
    \begin{definition}[Decomposed Wasserstein distance]
        We say that $W_d(\mu, \mu') \preceq (z_1, \ldots, z_k)$ if there exists a coupling $w \in \Gamma(\mu, \mu')$ such that, almost surely for $w \sim (x, x')$,
        \[
            \mathcal{G}(x,x') \preceq z.
        \]
    \end{definition}

    \begin{definition}[Decomposed shifted Rényi divergence]
        For any $z \in \mathbb{R}_{+}^{k}$, $q \ge 1$, and two distributions 
        $\mu,\nu$ defined on $\mathbb{R}^d$, we define
        \begin{equation}
            D_q^{(z)}(\mu \,\|\, \nu) 
            := \inf_{\mu' : W_d(\mu',\mu) \preceq z} 
               D_q(\mu' \,\|\, \nu).
            \label{eq:shifted-renyi}
        \end{equation}
    \end{definition}

The new divergence retains many properties of the original. In particular, with zero shift, it reduces to the standard R\'enyi divergence. Moreover, the Shift Reduction Lemma can be adapted to bound the divergence before and after adding Gaussian noise.

\begin{lemma}[Decomposed Shift Reduction Lemma for Gaussians]
\label{lem:shift-reduction}
Let $q \ge 1$, $z,a \ge 0$, and $X,Y$ be arbitrary random variables and matrix $A_i$ as described in Definition~\ref{def:decomp-gap}.
If $\xi,\xi' \sim \mathcal{N}(0,\sigma^2 I_{r_i})$ with $\sigma > 0$, then
\begin{equation}
    D_q^{(z)}(X + A_i\xi \,\|\, Y + A_i\xi')
    \le
    D_q^{(z + ae_i)}(X \,\|\, Y) 
    + \frac{q a^2}{2\sigma^2}.
    \label{eq:lemma20-decomposed}
\end{equation}
\end{lemma}

\begin{proof}
The original proof for the unshifted case is adapted by modifying the first step of the inequality. In particular, instead of considering $(X + W, -W + \xi)$ and $(Y, \xi')$, we consider $(X + A_i W, -W + \xi)$ and $(Y, \xi')$, so as to ensure the shift $(0,\ldots,a,\ldots,0) = a e_i$. Indeed, $X + A_i\xi$ and $Y + A_i\xi'$ can be obtained from $(X + A_iW, -W + \xi)$ and $(Y, \xi')$ by post-processing $f(x, y) = x + A_iy$. 

For the non-zero shift $z$, 
we adapt the proof by redefining $W_1$. Rather than the original choice $W_1 = h_z(W)$ (with $h_z(x)=x$ if $|x|\le z$ and $h_z(x)=\tfrac{x}{|x|}z$ otherwise), we set
\[
    W_1 = \sum A_i h_{z_i}(B_i).
\]

In this case, we observe that $W_1$ satisfies $G(W_1, 0) \preceq z$.  
Moreover, $G(W, W_1) \preceq a e_i$ whenever $G(W, 0) \preceq z + a e_i$.  The remainder of the proof then follows directly from the original argument.
\end{proof}

\subsection{Subspace design strategies}
\label{subsec:applications}

The certification in Theorem~\ref{thm:block-guarantee} applies to any
admissible orthogonal decomposition, so the block construction can be chosen
independently of the proof. In our experiments, we use a layer-wise random
orthogonal decomposition: for each layer, we sample a Gaussian matrix,
orthogonalize it, and split the resulting basis into $k$ approximately equal
blocks. This gives balanced random blocks without relying on a hand-designed
architectural partition, but requires quadratic memory in the decomposed layer
dimension. Other admissible decompositions trade off structure and cost: random
permutations reduce this overhead to linear memory, while layer-wise grouping
has no additional projection cost but relies more strongly on architectural
structure.
For nearly equal blocks, isotropic noise decomposes into independent block-wise noises with the same scalar variance. Hence assigning
each block a R\'enyi budget $\varepsilon^{\mathrm{renyi}}/k$ preserves the same
overall $(\varepsilon,\delta)$ budget after composition. Balanced random blocks
also have update norms that scale roughly as $1/\sqrt{k}$, so the clipping
bounds scale approximately as $C_0/\sqrt{k}$ and $C_1/\sqrt{k}$. Thus, each
noisy update perturbs only one block while keeping the per-block certification
comparable to the baseline.

In the resulting protocol, we run $T$ noisy updates per block, for a total of
$kT$ block-wise noisy updates, followed by the same post-unlearning fine-tuning
phase as in NFT. Although this increases the number of noisy-update steps, this
phase is short in our experiments; the total runtime is largely dominated by
post-unlearning fine-tuning. Thus, for small $T$, the block-wise schedule stays
close to the baseline runtime while producing more stable and higher-accuracy
trajectories.

\textbf{Relation to fixed-subspace PEFT.}
Our method should not be viewed as fixed-subspace PEFT. Since forget-set information may be spread across all coordinates, permanently freezing a
complementary subspace would require an additional cleanliness argument. Our schedule instead updates one block at a time while covering the full parameter
space over the unlearning trajectory.

\begin{figure*}[t]
    \centering
    \includegraphics[width=1.\linewidth]{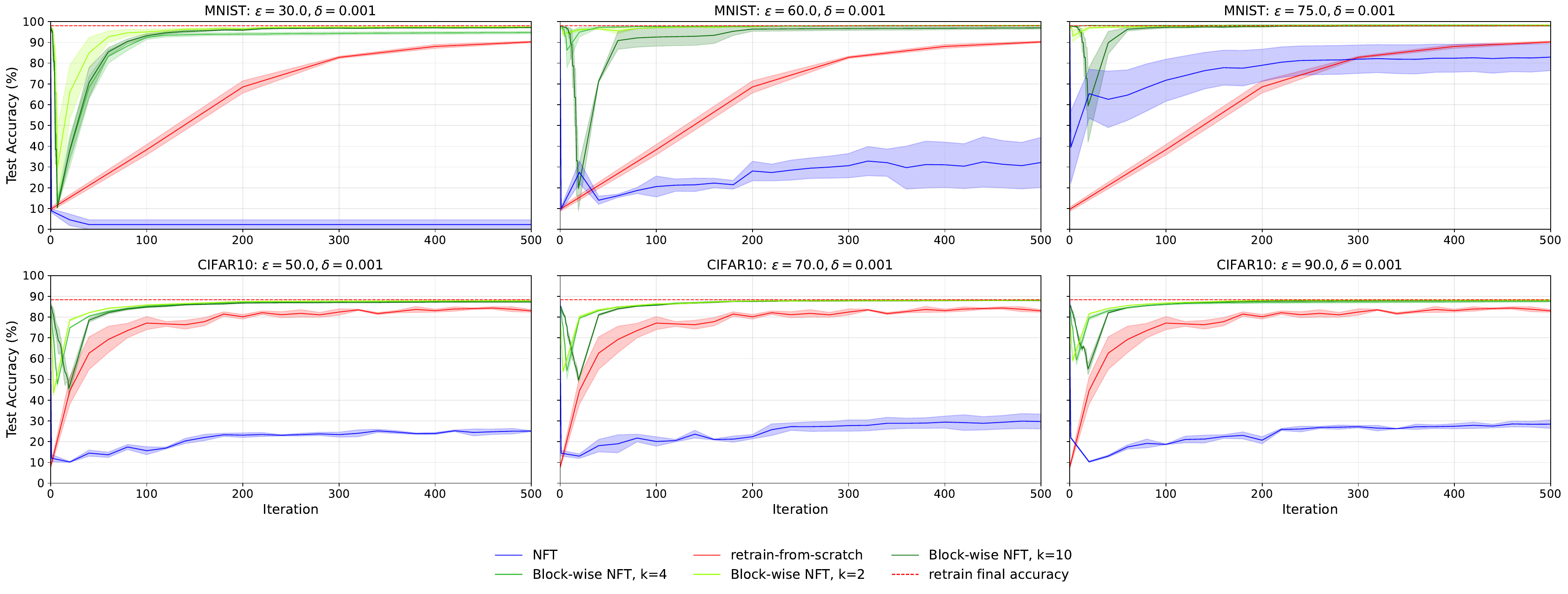}
    \caption{\textbf{Random 10\% deletion on \textsc{MNIST} and \textsc{CIFAR-10}.} We compare
    $\Delta(\rho)$-calibrated Noisy Fine-Tuning (NFT)
    with Block-wise NFT ($k=2,4,10$) with the final retrain accuracy shown for reference. Across privacy budgets Block-wise NFT shows smoother, more stable unlearning and better post–fine-tuning recovery.}
    \label{fig:block-nft-vs-nft}
\end{figure*}

\subsection{Calibration and re-certification}

\textbf{Calibrating $\Delta(\rho)$.}
We treat $\Delta(\rho)$ as a regime-level calibration parameter, fixed for a given model class, training pipeline, coupling of randomness, and deletion setting. In our experiments, we instantiate it empirically by running multiple coupled training pairs (10 in our experiments) with shared initialization, seeds, data order, and optimization hyperparameters. One run uses the original training stream, while the other uses a controlled perturbation of it: for random $10\%$ deletion, we remove a random
$10\%$ subset; for classwise deletion, we remove all examples of a target class. Removed samples are replaced by retained samples so that the mini-batch structure remains coupled across the two runs. We then measure the resulting parameter deviations and choose a conservative aggregate value, such as a high quantile, as $\Delta(\rho)$. Appendix~\ref{app:delta-estimation} provides the full calibration protocol, calibration plots, and an ablation on the sensitivity of Block-wise NFT to the chosen $\Delta(\rho)$.

This calibration is not universal: if deletion regime changes substantially, $\Delta(\rho)$ should be recalibrated. Such setting dependence is common in practical unlearning evaluations, where method-specific hyperparameters are typically selected for
the considered deletion scenario \citep{ebrahimpourboroojeny2025amunadversarialmachineunlearning}. We do not claim that our empirical instantiation is optimal; alternative couplings or theoretical bounds may yield smaller values and tighter budgets. Appendix~\ref{app:delta-upper-bound} discusses theoretical bounds based on argument stability.

\textbf{Re-certification.} Because $\Delta(\rho)$ is instantiated empirically, its estimate may change as more calibration evidence becomes available. Importantly, this does not necessarily require re-running unlearning: the optimization trajectory is fully determined by $(\lambda,\gamma,C_1,\sigma^2,T)$, and remains unchanged as long as these parameters are fixed. Thus, if a previously used estimate of $\Delta(\rho)$ is later deemed too optimistic, one can conservatively \emph{recalculate} the corresponding certified \emph{budget} for the same run without modifying it. We provide the formal statement and its proof in Appendix~\ref{app:budget-recalculation}.

\section{Experiments}
We empirically evaluate Block-wise NFT against retraining from scratch and a $\Delta(\rho)$-calibrated NFT baseline, in which \emph{the initial clipping is replaced by the discrepancy $\Delta(\rho)$}, then compare it with empirical unlearning methods using membership inference attack (MIA) auditing.

\textbf{Benchmarks, models, and scenarios.} We evaluate on \textsc{MNIST} \citep{lecun1998mnist} with a fully connected network of 4.36M parameters (architecture in Appendix~\ref{model:linear-arch}) and on \textsc{CIFAR-10} \citep{krizhevsky2009cifar10} using standard vision backbones: ResNet-18~\citep{he2016resnet} (commonly used in unlearning benchmarks) and a transformer-based architecture (ViT-Tiny)~\citep{DBLP:journals/corr/abs-2010-11929} to verify that our method applies beyond CNNs.
We consider two deletion settings: \emph{random 10\%} and \emph{classwise} deletion, with additional results in Appendix~\ref{app:experiments}. All remaining hyperparameters are selected via a small grid search over a predefined set on a held-out split of $\mathcal{D}_r$. Details are in Appendix~\ref{app:hyperparams}; code is in Appendix~\ref{app:code}. Although our experiments use this simple empirical selection procedure, Appendix~\ref{app:hyperparameter-selection} shows how the certificate can also provide a theory-motivated initialization for $\sigma^2$ and $C_1$. 

\textbf{Procedure.}
For NFT-based methods, we fix $(\varepsilon,\delta)$ per plot with $\delta=10^{-5}$ and vary $\varepsilon$ across plots. For each setting, we compute the minimal feasible noise scale $\sigma$ (Theorem~\ref{thm:noise-lb}), derive the corresponding step budget $T$, and run unlearning for $T$ steps (Block-wise NFT applies this sequentially across blocks), followed by standard fine-tuning.
We cap the \emph{total} number of unlearning+fine-tuning iterations at 1000 ($\le$1.5 epochs at batch size 64 on 90\% of the data). In practice, methods typically reach their peak accuracy well before the end of fine-tuning. We use \emph{random blocks}; other designs are deferred to the Appendix. Results are averaged over 5 runs to ensure statistical reliability.

\subsection{Block-wise NFT vs.\ NFT}

Figure~\ref{fig:block-nft-vs-nft} shows results for the random 10\% deletion task.
Block-wise NFT is consistently more stable: the accuracy drop during the unlearning phase is smaller, and recovery under fine-tuning is stronger. For instance, at the tighter budgets on \textsc{MNIST}, NFT fails to recover even after fine-tuning, while Block-wise NFT retains non-trivial accuracy.

The trajectories are smoother, suggesting that distributing noise across subspaces better preserves retained data. Varying the number of blocks ($k\in{2,4,10}$) yields comparable behavior, indicating that the gains are not strongly sensitive to $k$ in this range. Additional ablations over $k$, block construction, runtime, and sensitivity to the instantiated $\Delta(\rho)$ are provided in Appendix~\ref{app:ablations}.

Our main baseline is standard NFT, as our objective is to improve certified
unlearning without giving up formal guarantees. Empirical methods are included
only as auxiliary reference points under standard unlearning metrics and MIA
auditing; they are not directly comparable to certified methods and are not the
primary target of our comparison.

\subsection{Auxiliary comparison to empirical methods}\label{sec:approx-baselines}

\begin{table}[t]
\begin{center}
\caption{\textbf{Class 5 deletion on \textsc{CIFAR-10}.} 
Reported metrics: UA, RA, TA, MIA, RTE. Block-wise NFT matches Retrain on UA and MIA, while remaining competitive on RA and TA. Baseline results are taken from the official SaLUN repository~\citep{fan2024salun}. Empirical baselines are shown as auxiliary reference points; unlike Block-wise NFT,
they do not provide certified unlearning guarantees.
} 

\label{tab:emp-comparison}
\scalebox{0.75}{
\begin{tabular}{lccccc}
\multicolumn{1}{c}{\bf Method} & \multicolumn{1}{c}{\bf UA} & \multicolumn{1}{c}{\bf RA} & \multicolumn{1}{c}{\bf TA} & \multicolumn{1}{c}{\bf MIA} & \multicolumn{1}{c}{\bf RTE} 
\\ \toprule 
Retrain & 
100.00 (+0.00) & 
100.00 (+0.00) & 
86.14 (+0.00) & 
100 & 
46.37 \\
\midrule
FT  & 
47.43 ($-$52.57) & 
99.96 ($-$0.04) & 
95.57 (+9.43) & 
47.4 & 
2.6 \\
GA  & 
94.09 ($-$5.91) & 
92.20 ($-$7.80) & 
87.03 (+0.89) & 
94.09 & 
0.15 \\
IU  & 
98.98 ($-$1.02) & 
98.18 ($-$1.82) & 
93.42 (+7.28) & 
98.98 & 
0.5 \\
SalUN  & 
100.00 (+0.00) & 
99.81 ($-$0.19) & 
95.10 (+8.96) & 
100 & 
2.76 \\
$\ell_1$-sparse & 
100.00 (+0.00) & 
91.79 ($-$8.21) & 
89.08 (+2.94) & 
100 & 
2.55 \\
AMUN & 
100.00 (+0.00) & 
99.32 ($-$0.68) &
85.79 (+1.65) & 
100 & 
2.86 \\
\midrule
Block-wise NFT & 
100.00 (+0.00) & 
96.18 ($-$3.82) & 
83.37 ($-$2.77) & 
100 & 
0.85 \\
\end{tabular}
}
\end{center}

\end{table}

Table~\ref{tab:emp-comparison} reports results for the task of forgetting class~$5$; results for other target classes are provided in Appendix~\ref{app:classwise-exps}. We compare Block-wise NFT to retraining, fine-tuning (FT)~\citep{warnecke2023machineunlearningfeatureslabels}, gradient ascent (GA)~\citep{thudi2022unrollingsgdunderstandingfactors}, influence unlearning (IU)~\citep{pmlr-v70-koh17a}, AMUN~\citep{ebrahimpourboroojeny2025amunadversarialmachineunlearning} and sparsity-based approaches (SaLUN~\citep{fan2024salun}, $\ell_2$-sparsity~\citep{jia2023model}). Baselines were reproduced using the official SaLUN repository.

Evaluation uses standard metrics: unlearned accuracy (UA, $1-\text{Acc}(\mathcal{D}_f)$), test and retain accuracy (TA and RA), membership inference attack score (MIA)~\citep{jia2023model} (see Appendix~\ref{app:mia} for details), and run-time efficiency (RTE, minutes), with all results interpreted relative to retraining.

Block-wise NFT achieves UA=100\% and MIA=100\%, matching retraining on these
forgetting metrics. This is consistent with the certified guarantee and
shows no evidence of leakage under this audit. While some empirical methods obtain
higher utility, they do not provide certified unlearning guarantees. Among methods
with $UA=100\%$, Block-wise NFT also has a low training cost (RTE).

\section{Conclusion}

We studied utility limitations in perturbation-based certified unlearning. In particular, we showed that, for a predefined number of unlearning steps, the standard clipping-based analysis of NFT is overly conservative: it creates a tension between maintaining accuracy and improving over full retraining. This motivates our reformulation in terms of the proximity parameter $\Delta(\rho)$, which captures the practically relevant distance between the full-data model and its retained-data retraining counterpart.

We further proposed Block-wise Noisy Fine-Tuning, which applies certified noisy updates sequentially across orthogonal parameter blocks. Under our instantiated calibration of $\Delta(\rho)$, this block-wise schedule stabilizes unlearning and substantially reduces the transient accuracy drop relative to standard NFT, while remaining consistent with membership-inference auditing. Overall, our results suggest that NFT-style certified unlearning can be made substantially less destructive in the regimes we study, while retaining formal certified-unlearning guarantees.

More broadly, our findings suggest that worst-case certified-unlearning analyses can be misaligned with the practical comparison between full-data and retained-data training. Future work should develop certificates that better reflect training dynamics while retaining formal guarantees, and further study how calibration, block decompositions, and auditing protocols interact across different deletion regimes.

\section*{Acknowledgements}

We gratefully acknowledge funding from the European Research Council (ERC) under the Horizon Europe Framework Programme (HORIZON) for proposal number 101170430 CollectiveMinds. Views and opinions expressed are however those of the authors only and do not necessarily reflect those of the European Union or the European Research Council. Neither the European Union nor the granting authority can be held responsible for them.
The authors would like to thank Anastasia Koloskova for helpful discussions.

\bibliography{references}

\newpage
\appendix
\onecolumn

\section{Related work}

\paragraph{Certified Machine Unlearning.} Certified machine unlearning has emerged as a practical alternative to full retraining, offering rigorous data-removal guarantees at much lower cost. 
Classical approaches estimate the retrained solution via single–step Newton and influence surrogates (\cite{pmlr-v119-guo20c,SekhariAKS21,zhang2024towards}), or use projected/perturbed gradient methods (\cite{pmlr-v132-neel21a, chien2024langevin}), coupled with randomized mechanisms to ensure statistical indistinguishability. Many of these certificates are explicitly \emph{DP-inspired}, applying noisy updates and privacy amplification ideas from differential privacy \cite{10.1561/0400000042, pmlr-v108-balle20a}. However, recent methods typically rely on strong assumptions (e.g., smoothness/strong convexity, knowledge of Hessian eigenvalues, or even a unique minimizer), which limit applicability to modern deep networks \cite{zhang2024towards, liu2023certified,allouah2025the}. In this work we remain within the certified DP line, but we specifically adapt the noisy fine-tuning framework of \cite{koloskova2025certified} to preserve certificates \emph{without} convexity or uniqueness assumptions while mitigating the accuracy loss usually observed in DP-based unlearning.

\paragraph{Group-wise and layer-wise DP optimization.}
Our construction is related to group-wise and layer-wise DP-SGD methods \citep{he2022exploringlimitsdifferentiallyprivate, nguyen2023batchclippingadaptivelayerwise}, which split parameters or gradients into groups to reduce clipping bias or exploit layer heterogeneity in private training. The setting is different: these methods train privately from a clean initialization, whereas certified unlearning starts from a model already trained on the forget set and must certify closeness to retained-data retraining. Thus, we use block structure not only as an optimization device, but as a sequential certified-unlearning schedule whose privacy analysis composes across blocks while covering the full parameter space.

\paragraph{Empirical unlearning in Image Classification.} Empirical (approximate) unlearning methods aim to remove the influence of the forget set without formal certificates and are therefore evaluated using strong auditing attacks alongside simple utility/efficiency metrics. Representative methods include sparsity-driven approaches \citep{fan2024salun,jia2023model}, fine-tuning on the retain set \citep{WarPirWreRie20}, gradient ascent on the forget set \citep{thudi2022unrollingsgdunderstandingfactors}, and adversarial-example-based unlearning \citep{ebrahimpour-boroojeny2025not}, among others. Efficacy is typically assessed with unlearning-specific membership-inference attacks (MIAs) \citep{kurmanji2023towards,jia2023model, ebrahimpour-boroojeny2025not}, together with standard model-performance metrics. Although these methods provide no formal guarantees, they constitute strong practical baselines and useful evaluation tools for auditing certified approaches.

\section{Code availability}\label{app:code}

We release our implementation at \url{https://github.com/mlolab/blockwise-noisy-fine-tuning}. The repository includes the code needed to reproduce our experiments.

\section{Thought experiment formalization}

\subsection{Proof of Proposition~\ref{prop:thought-exp}.}
\label{app:dp-perf}


\begin{proof}
    
In the worst-case, initialization-agnostic certification, for any measurable set $O\subseteq\mathbb{R}^d$ and any two initializations (in particular, $\hat{\mathbf{x}}$ and $\mathbf{x}_{\mathrm{init}}$), we have
\begin{equation}
\label{eq:init-agnostic}
\begin{aligned}
\Pr\!\big[\mathcal{U}(\hat{\mathbf{x}})\in O\big] &\le e^{\varepsilon}\Pr\!\big[\mathcal{U}(\mathbf{x}_{\mathrm{init}})\in O\big] + \delta,\\
\Pr\!\big[\mathcal{U}(\mathbf{x}_{\mathrm{init}})\in O\big] &\le e^{\varepsilon}\Pr\!\big[\mathcal{U}(\hat{\mathbf{x}})\in O\big] + \delta.
\end{aligned}
\end{equation}

Applying the second inequality in~\eqref{eq:init-agnostic} to $O=\mathcal{E}_{\mathrm{fail}}$ yields
\[
\Pr\!\big[\mathcal{U}(\mathbf{x}_{\mathrm{init}})\in \mathcal{E}_{\mathrm{fail}}\big]
\le e^{\varepsilon}\Pr\!\big[\mathcal{U}(\hat{\mathbf{x}})\in \mathcal{E}_{\mathrm{fail}}\big] + \delta .
\]

Let $\mathcal{E}_{\mathrm{full}}$ denote the set of possible outcomes of $\mathcal{U}(\hat{\textbf{x}})$.
Applying the second inequality in~\eqref{eq:init-agnostic} to $O = \mathbb{R}^d \setminus \mathcal{E}_{\mathrm{full}}$ yields
\[
\Pr\!\big[\mathcal{U}(\mathbf{x}_{\mathrm{init}})\in \mathbb{R}^d \setminus \mathcal{E}_{\mathrm{full}}\big]
\le e^{\varepsilon}\Pr\!\big[\mathcal{U}(\hat{\mathbf{x}})\in \mathbb{R}^d \setminus \mathcal{E}_{\mathrm{full}}\big] + \delta
= \delta.
\]

Therefore, the probability of obtaining bad performance (below the guarantee~$\alpha$) can be bounded by
\[
e^{\varepsilon}\Pr\!\big[\mathcal{U}(\hat{\mathbf{x}})\in \mathcal{E}_{\mathrm{fail}}\big] + 2\delta,
\]
which is almost $\Pr\!\big[\mathcal{U}(\hat{\mathbf{x}})\in \mathcal{E}_{\mathrm{fail}}\big]$ when $(\varepsilon,\delta)$ is tight.
\end{proof}

\subsection{Limitations and extensions}
\label{app:dp-perf-2}

\paragraph{Thought experiment limitations.}
The argumentation of our thought experiment implicitly relies on the assumption that the number of steps required to retrain from scratch to the original performance guarantee $(\alpha,p)$ is comparable to the number of steps required to reach the slightly weaker guarantee $(\alpha,\, 1 - e^{\varepsilon}(1-p) - 2\delta)$. While this is a natural assumption when $(\varepsilon,\delta)$ are small and the certified guarantee is tight, it need not hold universally. In particular, for some models or training regimes, small degradations in success probability may correspond to substantially different retraining times. Our thought experiment should therefore be interpreted as highlighting a structural tension induced by worst-case, initialization-agnostic guarantees, rather than as a strict impossibility result under all settings.

\paragraph{Application to other training-agnostic algorithms.}
Note that the argument uses only that the unlearning procedure is training-agnostic and that the number of steps $T$ is fixed in advance, i.e., it does not depend on the model weights.

We stated the argument in terms of step counts $T$ and $T_{\mathrm{retrain}}$, although the per-step computational cost may differ across procedures. The same reasoning can be reformulated in terms of wall-clock time or total compute, provided the algorithm's computation cost does not depend on the model weights.

\section{Theoretical bounds on $\Delta(\rho)$ via argument stability}
\label{app:delta-upper-bound}

The key quantity in our certificates is the discrepancy $\Delta(\rho)$ between the full retrain and the retraining-from-scratch. While obtaining a tight bound on $\Delta(\rho)$ for modern deep models in full generality is challenging, this question is well-studied in the literature on \emph{argument stability} (also known as parameter stability). Below we summarize regimes where such bounds are known or can be derived under standard assumptions.

Algorithmic stability has long been used to analyze how sensitive
learning algorithms are to changes in the training data. \cite{pmlr-v70-liu17c} formalized \emph{argument stability}, which upper-bounds the parameter deviation between hypotheses trained on neighboring datasets. Several subsequent works derive explicit bounds on this deviation under standard assumptions.

\paragraph{Smooth and Lipschitz losses.}
\cite{pmlr-v48-hardt16} provide a trajectory-level recursion for SGD under
$\beta$-smooth and $L$-Lipschitz losses in three regimes:
(i) non-convex,  
(ii) convex,  
(iii) $\gamma$-strongly convex.  
When the replaced index is known --- as in our full-vs-retained setting --- their
recursion yields a closed-form upper bound on
\[
\|\theta^{(t)} - \theta'^{(t)}\|,
\]
and hence on our proximity $\Delta(\rho)$.

\paragraph{Example (strongly convex case).}
For completeness, we briefly illustrate how a bound on $\Delta(\rho)$
follows from the recursion of \cite{pmlr-v48-hardt16}. Suppose the loss is $\gamma$-strongly convex and $\beta$-smooth, and SGD uses a constant stepsize $\alpha \le 1/\beta$. Let $\theta^{(t)}$ and $\theta'^{(t)}$ denote the SGD iterates obtained from the full and retained datasets, respectively, and let $\delta_t = \|\theta^{(t)} - \theta'^{(t)}\|$. Since both runs start from the same initialization, we have $\delta_0 = 0$.

\cite{pmlr-v48-hardt16} give the following one-step inequalities:

\begin{enumerate}
    \item If the minibatches coincide:
    \[\delta_{t+1} \;\le\; (1 - \alpha\gamma)\,\delta_t.\]
    \item If the minibatches differ: 
    \[\delta_{t+1} 
    \;\le\; (1 - \alpha\gamma)\,\delta_t + 2\alpha L.\]
\end{enumerate}

Unrolling this recursion yields the explicit bound
\[
\delta_T 
\;\le\; 2\alpha L 
\sum_{k \in B} (1 - \alpha\gamma)^k,
\]
where $B$ is the set of iteration indices for which the minibatches differ. This provides a closed-form upper bound on $\delta_T$, and therefore on $\Delta(\rho)$ in this regime.

\paragraph{Nonsmooth convex losses.}
\cite{bassily2020stabilitystochasticgradientdescent} prove argument stability for SGD without requiring smoothness. Using the monotonicity of subgradients together with $L$-Lipschitzness, they bound the deviation $\|\theta^{(t)} - \theta'^{(t)}\|$ at every iteration, directly giving an
upper bound on $\Delta(\rho)$.

\paragraph{Neural networks.}
For certain neural architectures, stability of gradient-based methods has also been established. The work of \cite{NEURIPS2022_fb8fe6b7} proves stability-based generalization bounds for shallow ReLU networks. Complementarily, \cite{NEURIPS2021_483101a6} analyze the dynamics of gradient descent and obtain stability and generalization guarantees for shallow networks beyond the NTK regime. In both cases, the analysis controls how the learned parameters change under small perturbations of the training data, providing architecture-specific upper bounds on the parameter deviation and hence on $\Delta(\rho)$ for these models.

Together, these results show that $\Delta(\rho)$ can be bounded in several well-understood regimes -- including smooth non-convex objectives, nonsmooth convex objectives, and specific neural-network architectures for which parameter-stability analyses exist. A general closed-form bound is not available for all training pipelines and architectures, but many practical settings already fall into one of the regimes above.

\section{Notations}

We summarize the main notation used throughout the paper in Table~\ref{table:notations}.

\begin{table}[H]
\caption{Notation used throughout the paper}
\label{table:notations}
\begin{center}
\begin{tabular}{ll}
\multicolumn{1}{c}{\bf Symbol}  &\multicolumn{1}{c}{\bf Description}
\\ \hline \\
$\mathcal{A}$            & Learning algorithm mapping a dataset to model parameters \\
$\mathcal{D}$            & Training dataset
\\
$d$                       & Dimension of the parameter (weight) space $\mathbb{R}^d$ \\
$\hat{\mathbf{x}}$       & Model parameters obtained from $\mathcal{A}(\mathcal{D})$ \\
$\mathcal{D}_f$          & Subset of $\mathcal{D}$ to be forgotten \\
$\mathcal{D}_r$          & Retained dataset, $\mathcal{D} \setminus \mathcal{D}_f$ \\
$\mathcal{U}$            & Unlearning mechanism updating parameters after deletion request \\
$\tilde{\mathbf{x}}$     & Model parameters output by $\mathcal{U}$ \\
$\bar{\mathcal{A}}$      & Certifying algorithm in the definition of $(\varepsilon,\delta)$-unlearning \\
$\varepsilon$            & Privacy/unlearning parameter controlling multiplicative slack \\
$\delta$                 & Privacy/unlearning parameter controlling additive slack \\
$\mathbf{x}_0$           & Model parameters after initial clipping of $\hat{\mathbf{x}}$ with radius $C_0$ \\
$\mathbf{x}_t$           & Model parameters at iteration $t$ of noisy fine-tuning \\
$g_t$                    & Gradient at step $t$ computed on the retained data $\mathcal{D}_r$ \\
$\boldsymbol{\xi}_{t}$   & Gaussian noise $\sim \mathcal{N}(0,\sigma^2 I_d)$ added at step $t$ \\
$\Pi_{C}$                & Clipping operator with radius $C$, $\Pi_C(\mathbf{v}) = \mathbf{v}\cdot\min\{\tfrac{C}{\|\mathbf{v}\|},1\}$ \\
$\hat{\mathbf{x}}'$      & Model parameters obtained by training on the retained dataset $\mathcal{D}_r$ \\
$\mathbf{x}'_t$          & Iterates produced by unlearning updates when initialized at $\hat{\mathbf{x}}'$ \\
$\sigma^2$               & Variance of the Gaussian noise in noisy fine-tuning \\
$T_{\mathrm{retrain}}$   & Minimal number of retraining steps required to reach good accuracy on $\mathcal{D}_r$\\
$T$                      & Number of unlearning steps in NFT \\
$\alpha$                 & Target accuracy level in the thought experiment \\
$p$                      & Probability of achieving accuracy at least $\alpha$ after $T$ steps \\
$\mathbf{x}_{\mathrm{init}}$ & Randomly initialized model parameters \\
$\Delta(\rho)$           & High-probability bound on the distance between $\hat{\mathbf{x}}$ and $\hat{\mathbf{x}}'$ \\
$\rho$                   & Failure probability in the definition of $\Delta(\rho)$ \\
$q$                        & Order of Rényi Differential Privacy \\
$\varepsilon^{\text{r\'enyi}}$ & Privacy loss at order $q$ in Rényi DP \\
$\sigma^2$                 & Variance of Gaussian perturbations in noisy fine-tuning \\
$T(\sigma^2)$              & Minimal number of unlearning steps required for a given noise variance $\sigma^2$ \\
$\sigma^2_{\min}$          & Minimal feasible noise variance ensuring $(\varepsilon,\delta)$-unlearning \\
$k$                 & Number of subspaces (blocks) in the partition of $\mathbb{R}^d$ \\
$r_i$               & Dimension of the $i$-th subspace, with $\sum_{i=1}^k r_i = d$ \\
$A_i \in \mathbb{R}^{d \times r_i}$ & Matrix with orthonormal columns spanning subspace $V_i$ \\
$A$                 & Concatenation $[A_1 \cdots A_k] \in \mathbb{R}^{d \times d}$ with $A^\top A = I_d$ \\
$B_i$               & Coordinate vector in $\mathbb{R}^{r_i}$ corresponding to $V_i$ in the decomposition of $W$ \\
$\mathcal{G}(W,W')$       & Decomposition gap between two weight vectors \\
$z^{(1)} \preceq z^{(2)}$ & Coordinate-wise inequality between two vectors in $\mathbb{R}^k$ \\
$W_d(\mu,\mu')$           & Decomposed Wasserstein distance between distributions $\mu,\mu'$ \\
$D_q^{(z)}(\mu \,\|\, \nu)$ & Decomposed shifted Rényi divergence with shift vector $z$ and order $q$ \\
$e_i$                     & $i$-th standard basis vector in $\mathbb{R}^k$ \\

\end{tabular}
\end{center}
\end{table}

\section{Proofs}

\label{app:theory-minor}

\begin{proof}[Proof of uniqueness of the decomposition]
Since $A = [A_1 \;\cdots\; A_k]$ is orthogonal, we have $A^\top A = I_d$. 
For any $W \in \mathbb{R}^{d}$ set $B = A^\top W$ 
and partition it into blocks $B^\top = [B_1^\top, \dots, B_k^\top]$. 
Then
\[
AB = AA^\top W = W,
\]
which expands to $W = \sum_{i=1}^k A_i B_i$. 
Uniqueness follows from orthogonality: if $\sum_i A_i B_i = \sum_i A_i B_i'$, then $A^\top ( \sum_i A_i(B_i - B_i')) = \sum_i B_i - B_i' = 0$, hence $B_i = B_i'$.
\end{proof}

\subsection{Proof of Theorem \ref{thm:noise-lb} \label{app:lower-bound}}

We first recall the shifted R\'enyi divergence \citep{Feldman_2018}:

\begin{definition}[Rényi divergence]\label{def:renyi}
Let $q > 0$, $q \ne 1$. The $q$-Rényi divergence between two probability distributions $\mu$ and $\nu$ is defined as
\[
D_q(\mu \,\|\, \nu) := \frac{1}{q - 1} \log \mathbb{E}_{X \sim \nu} \left( \frac{\mu(X)}{\nu(X)} \right)^q .
\]
\end{definition}

Let us consider the reasoning presented in \citep{koloskova2025certified}. Their proof contains a minor indexing mismatch in the recursion expansion, which did not affect the final result for their asymptotics, but it will be important for us. We corrected the indices in the product $p_t$ accordingly.

Let us revisit the reasoning in \citep{koloskova2025certified}. In their proof, the transition from equation (23) to equation (24) involves a minor computational error in solving the recursion. While this did not affect their asymptotic conclusions, it becomes relevant for our analysis. We correct this step by adjusting the indices in the product $p_t$ accordingly. The corrected parts are highlighted in \textcolor{green!60!black}{green}.

\begin{theorem}[Koloskova A.9, fixed indices in recursion]
    Let $T,q \geq 1$, $\gamma_0,\ldots,\gamma_{T-1} \geq 0$, 
    $\sigma_0,\ldots,\sigma_{T-1} > 0$, $\lambda \geq 0$ and consider the two sequences 
    $\{x_t\}_{0 \leq t \leq T}, \{x'_t\}_{0 \leq t \leq T}$ as defined above. 
    Denote by $D_q$ the R\'enyi divergence of order $q$. 
    Assume that for every $t \in \{0,\ldots,T-1\}$, $\gamma_t \lambda < 1$. 
    Denote for every $t \in \{0,\ldots,T-1\}$,
    \begin{equation}        
        s_t := 2 \gamma_t C_1, 
        \qquad 
        \rho_t := 1 - \gamma_t \lambda.
    \end{equation}
    If there exist $a_0,\ldots,a_{T-1} \in\mathbb{R}_{\ge 0}$ such that
    \begin{equation}        
        \sum_{t=0}^{T-1} 
        \Biggl( \prod_{k=\textcolor{green!60!black}{t + 1}}^{\textcolor{green!60!black}{T-1}} \rho_k \Biggr) a_t
        =
        \Biggl( \prod_{t=0}^{T-1} \rho_t \Biggr) 2 C_0
        + \sum_{t=0}^{T-1} 
        \Biggl( \prod_{k=\textcolor{green!60!black}{t + 1}}^{\textcolor{green!60!black}{T - 1}} \rho_k \Biggr) s_t,
    \end{equation}
    then
    \begin{equation}  
        D_q(x_T \,\|\, x'_T) \;\;\leq\;\;
        \sum_{t=0}^{T-1} \frac{q a_t^2}{2 \sigma_t^2}.
    \end{equation}
    \end{theorem}
    Let us attempt to find the optimal parameters. We will search for a solution under the following constraints:
\begin{itemize}

    \item Suppose there exists a maximum allowable noise level at any given step, $\sigma \ge \sigma_i$.

    \item Suppose the weight decay $\lambda$ and learning rate $\gamma$ are the same for all steps.

    \item We will also seek a solution for fixed $C_0$, $C_1$, and RDP-budget $(\varepsilon^{\text{r\'enyi}}, q)$.

\end{itemize}

For simplicity, let us denote $\alpha_t = \prod_{k=t + 1}^{T-1} \rho_k$, then the condition on $a$ can be rewritten as

    \begin{equation}
        \sum_{t=0}^{T-1} 
        \alpha_t a_t
        =
        \rho_0 \alpha_0 2 C_0
        + \sum_{t=0}^{T-1} 
        \alpha_t s_t,
    \label{eq:unlearning-condition}
    \end{equation}

  In fact, determining $T$ is equivalent to finding the first moment when the left-hand side exceeds the right-hand side:
  \begin{equation}
        \sum_{t=0}^{T-1} 
        \alpha_t a_t
        \;\;\ge\;\;
        \rho_0 \alpha_0 2 C_0
        + \sum_{t=0}^{T-1} 
        \alpha_t s_t.
        \label{eq:T-condition}
    \end{equation}

Moreover, the unlearning budget $\varepsilon$ is computed directly from the constraint on the Rényi divergence:
\begin{equation}
    D_q(x_T \,\|\, x'_T) 
    \;\;\leq\;\;
    \sum_{t=0}^{T-1} \frac{q a_t^2}{2 \sigma_t^2} 
    \;\;\le\;\; \varepsilon^{\text{r\'enyi}}.
    \label{eq:renyi-budget}
\end{equation}

Let $y$ be the vector with coordinates $y_i = \tfrac{a_i}{\sigma_i}$. By the constraint, we are searching for a solution such that
\begin{equation}
    \|y\|^2_2 \;\le\; \tfrac{2\varepsilon^{\text{r\'enyi}}}{q}.
    \label{eq:y-constraint}
\end{equation}

Observe that by proportionally increasing $a_i$ until the inequality becomes an equality, we do not increase the minimal number of epochs $T$ required for condition~\eqref{eq:T-condition}. Hence, at the optimum we may assume
\begin{equation}
    \|y\|_2 = \sqrt{\tfrac{2\varepsilon^{\text{r\'enyi}}}{q}}.
    \label{eq:y-optimal}
\end{equation}

Next, consider the optimal choice of $a_i$ for our condition. By the Cauchy--Schwarz inequality,  
\begin{equation}
    \sum_{t=0}^{T-1} 
    \alpha_t a_t 
    \;\;\le\;\; 
    \sum_{t=0}^{T-1} 
    (\alpha_t \sigma_t)\,\Bigl(\tfrac{a_t}{\sigma_t}\Bigr) 
    \;\;\le\;\; 
    \|\alpha \sigma\| \cdot \|y\| 
    \;\;\le\;\; 
    \sqrt{\tfrac{2\varepsilon^{\text{r\'enyi}}}{q}} 
    \sqrt{\sum_{i=0}^{T-1} \alpha_i^2 \sigma_i^2}.
    \label{eq:cauchy-bound}
\end{equation}
Equality holds when the two vectors are proportional, which determines the optimal values of $a_i$.

Rewriting inequality~\eqref{eq:T-condition}, we obtain
\begin{equation}
    \rho_0 \alpha_0 2 C_0 
    + \sum_{t=0}^{T-1} \alpha_t s_t
    \;\;\le\;\;
    \sqrt{\tfrac{2\varepsilon^{\text{r\'enyi}}}{q}} 
    \sqrt{\sum_{i=0}^{T-1} \alpha_i^2 \sigma_i^2}
    \;\;\le\;\;
    \sqrt{\tfrac{2\varepsilon^{\text{r\'enyi}}\sigma^2}{q}} 
    \sqrt{\sum_{i=0}^{T-1} \alpha_i^2}.
    \label{eq:final-inequality}
\end{equation}

Thus, $T$ does not decrease if we set $\sigma_i = \sigma$ for every step.

For convenience, define
\begin{equation}
    C_b := \sqrt{\tfrac{2\varepsilon^{\text{r\'enyi}}\sigma^2}{q}}.
    \label{eq:Cb}
\end{equation}

We now expand the inequality we aim to obtain, using the fact that
\begin{equation}
    \alpha_t = \prod_{k=t+1}^{T-1} \rho_k = (1 - \gamma \lambda)^{T - 1 - t}.
\end{equation}

The left-hand side is
\begin{equation}
\begin{aligned}
    &\; 2 C_0 (1 - \gamma \lambda)^T 
    + 2 \gamma C_1 \sum_{t=0}^{T-1} (1 - \gamma \lambda)^{T - 1 - t} \\
    &= 2 C_0 (1 - \gamma \lambda)^T 
    + 2 \gamma C_1 \tfrac{1 - (1 - \gamma \lambda)^T}{\gamma \lambda},
\end{aligned}
\label{eq:LHS}
\end{equation}
while the right-hand side is
\begin{equation}
    C_b \sqrt{\sum_{t=0}^{T-1} (1 - \gamma \lambda)^{2T - 2 - 2t}}
    \;=\;
    C_b \sqrt{\tfrac{1 - (1 - \gamma \lambda)^{2T}}{1 - (1 - \gamma \lambda)^2}}.
    \label{eq:RHS}
\end{equation}
Observe that both sides of the inequality are positive, hence it suffices to require that the square of the left-hand side exceeds the square of the right-hand side.

Introduce the variable
\begin{equation}
    x = (1 - \gamma \lambda)^T.
    \label{eq:x-def}
\end{equation}
Then $T$ can be recovered from $x$ as
\begin{equation}
    T = \tfrac{\log(x)}{\log(1 - \gamma \lambda)}.
    \label{eq:T-from-x}
\end{equation}
Thus, the problem of finding the minimal $T$ is equivalent to finding the \textit{maximal} $x$, subject to the constraint $0 < x \le 1$.

By rewriting the difference of squares of the left- and right-hand sides, the problem reduces to finding the maximal root in the half-interval $(0,1]$ of the equation
\begin{equation}
    \Bigl( 2 C_0 x + 2 \gamma C_1 \tfrac{1 - x}{\gamma \lambda} \Bigr)^2 
    - C_b^2 \cdot \tfrac{1 - x^2}{1 - (1 - \gamma \lambda)^2} = 0.
    \label{eq:root-equation-1}
\end{equation}

Equivalently,
\begin{equation}
    \Biggl( 
        \Bigl(\tfrac{2C_0}{C_b} - \tfrac{2 C_1}{\lambda C_b}\Bigr) x 
        + \tfrac{2 C_1}{\lambda C_b} 
    \Biggr)^2
    - \tfrac{1 - x^2}{1 - (1 - \gamma \lambda)^2} = 0.
    \label{eq:root-equation-2}
\end{equation}

Define
\begin{equation}
    \zeta := \tfrac{1}{1 - (1 - \gamma \lambda)^2}, 
    \qquad 
    \beta_0 := \tfrac{2C_1}{ \lambda C_b}, 
    \qquad 
    \beta_1 := \tfrac{2C_0}{C_b} - \tfrac{2 C_1}{\lambda C_b} = \beta_0\big(\tfrac{\lambda C_0}{C_1} - 1\big).
    \label{eq:beta-defs}
\end{equation}

Then the quadratic equation can be written as
\begin{equation}
    (\beta_1^2 + \zeta)x^2 + 2 \beta_0 \beta_1 x + (\beta_0^2 - \zeta) = 0.
    \label{eq:quadratic}
\end{equation}

Its discriminant is
\begin{equation}
\begin{aligned}
    \Delta 
    &= 4 \beta_0^2 \beta_1^2 
       - 4 (\beta_1^2+\zeta)(\beta_0^2-\zeta) \\
    &= 4\zeta(\zeta + \beta_1^2 - \beta_0^2).
\end{aligned}
\label{eq:discriminant}
\end{equation}

The largest root of the quadratic is given by
\begin{equation}
    x_{\max} = 
    \frac{-2 \beta_0 \beta_1 + \sqrt{\,4\zeta(\zeta + \beta_1^2 - \beta_0^2)\,}}
         {2(\beta_1^2+\zeta)}.
    \label{eq:x-max}
\end{equation}

We now state the condition for the existence of a root:
\begin{equation}
    \zeta \;\ge\; \beta_0^2 - \beta_1^2 
    = \beta_0^2 - \Bigl(\tfrac{2C_0}{C_b} - \beta_0\Bigr)^2
    = \tfrac{8C_0C_1}{C_b^2 \lambda} - \tfrac{4C_0^2}{C_b^2}.
    \label{eq:root-condition-1}
\end{equation}

Equivalently,
\begin{equation}
    \tfrac{1}{1 - (1 - \gamma \lambda)^2} 
    \;\ge\;
    \tfrac{8C_0C_1}{C_b^2 \lambda} - \tfrac{4C_0^2}{C_b^2}
    = \tfrac{2q}{\sigma^2 \varepsilon^{\text{r\'enyi}}} 
      \Bigl(\tfrac{2}{\lambda} C_0 C_1 - C_0^2 \Bigr).
    \label{eq:root-condition-2}
\end{equation}

This implies
\begin{equation}
    \sigma^2 \;\ge\; \gamma(2 - \gamma \lambda) \cdot \tfrac{2 q}{\varepsilon^{\text{r\'enyi}}}
    \cdot \Bigl(2 - \tfrac{\lambda C_0}{C_1}\Bigr) \cdot C_0 C_1,
    \label{eq:sigma-bound}
\end{equation}

In the case $\tfrac{\lambda C_0}{C_1} \in (0, 1)$, the right-hand side of the inequality is positive, which yields a bound on the minimal $\sigma$ in this regime.

To obtain the bound on $T$, it suffices to substitute all original variables into the formula \ref{eq:x-max} for $x_{\max}$.

In the case of the minimal $\sigma$, the discriminant vanishes, and the expression simplifies to
\begin{equation}
\begin{aligned}
    x_{\max} 
    &= \frac{-\beta_0 \beta_1}{\beta_1^2+\zeta} 
     = \frac{-\beta_0 \beta_1}{\beta_0^2} 
     = \tfrac{-(2\lambda C_0 - 2C_1)}{2C_1} \\
    &= 1 - \tfrac{\lambda C_0}{C_1} \in (0,1].
\end{aligned}
\label{eq:xmax-minimal-sigma}
\end{equation}

Therefore, substituting into
\begin{equation}
    T = \tfrac{\log(x)}{\log(1 - \gamma \lambda)},
    \label{eq:T-bound}
\end{equation}
we obtain the desired estimate for $T$.

Furthermore, for $\tfrac{\lambda C_0}{C_1} \in [1, 2)$, the bound on the noise obtained from the discriminant is also positive; however, in this case, the corresponding $x_{\max}$ would fall outside the interval $(0, 1)$, as required.

To analyze the regime $\tfrac{\lambda C_0}{C_1} \ge 1$, we observe that $\beta_1 \le 0$, implying that the vertex of the parabola lies at a negative coordinate. Consequently, the value at $0$ must be negative (since the value at $1$ is positive). From the inequality $\zeta > \beta_0^2$, one therefore derives a bound on $\sigma^2$. It is worth noting, however, that in the case of equality, $x_{\max}$ becomes zero, which corresponds to an infinite $T$.

\paragraph{Formula for minimal $\sigma$ with given $T$}

Given argumentation above, we can express $\sigma$ in terms of $x_{\max}$:
\begin{equation}
    \text{Let} \quad 
    z := 1 - \frac{\lambda C_0}{C_1}, 
    \qquad 
    x := x_{\max} \in (0,1].
    \label{eq:z-def}
\end{equation}

We now rewrite the equation~\ref{eq:x-max} for $x_{\max}$ in terms of $\sigma$, which leads to a quadratic equation in $s := 1/\sigma^2$:
\begin{equation}
    a s^2 + b s + c = 0,
    \label{eq:quad-s}
\end{equation}
where
\begin{equation}
\begin{aligned}
    a &= \left(\frac{2q\,C_1^2}{\varepsilon^{\text{r\'enyi}}\,\lambda^2}\right)^{2}
         z^{2}\,(1 - x z)^{2}, \\[6pt]
    b &= \left(\frac{2q\,C_1^2}{\varepsilon^{\text{r\'enyi}}\,\lambda^2}\right)
         \frac{1}{\gamma\lambda(2-\gamma\lambda)}
         \Bigl[\,1 - 2xz + (2x^{2}-1)z^{2}\Bigr], \\[6pt]
    c &= \frac{x^{2}-1}{\bigl(\gamma\lambda(2-\gamma\lambda)\bigr)^{2}}.
\end{aligned}
\label{eq:abc}
\end{equation}

Observe that $a > 0$ and $c < 0$, hence we are interested in the largest root of the quadratic.

The (positive) solution for $\sigma^2$ can then be written as
\begin{equation}
    \sigma^{2}(x)
    = \frac{2a}{-\,b+\sqrt{\,b^{2}-4ac\,}}.
    \label{eq:sigma-solution-1}
\end{equation}

Or the equivalent expression
\begin{equation}
    \sigma^{2}(x)
    = \frac{-\,b-\sqrt{\,b^{2}-4ac\,}}{2c}.
    \label{eq:sigma-solution-2}
\end{equation}

\subsection{Proof of the theorem \ref{thm:block-guarantee}}
\label{app:certification-blocks}

The proof proceeds analogously to that of the original theorem \citep{koloskova2025certified}.

We adapt Lemma A7 from the original proof, using adapted definitions of $\infty$-Wasserstein distance and shifted R\'enyi divergence.
    
\begin{lemma}[Decomposed Lemma A7 \citep{koloskova2025certified}]
Let $q \ge 1$, $z,\rho,s \ge 0$, $\psi:\mathbb{R}^d \to \mathbb{R}^d$,  
and $X,Y$ be arbitrary random variables.  
If $\psi$ satisfies, for all $\mathbf{x},\mathbf{x}' \in \mathbb{R}^d$  
(for a single component $i$, while for the others nothing changes),  
\[
    \|\psi(\mathbf{x}') - \psi(\mathbf{x})\| 
    \;\le\; \rho \|\mathbf{x}' - \mathbf{x}\| + s,
\]
then
\[
    D_q^{(z_1, \ldots, (\rho z_i + s), \ldots, z_k)}(\psi(X) \,\|\, \psi(Y))
    \;\;\le\;\;
    D_q^{(z_1, \ldots, z_i , \ldots, z_k)}(X \,\|\, Y).
\]
\end{lemma}

\begin{proof}
To adapt the original proof, it suffices to observe that for the decomposed Wasserstein distance the following inequality holds:

from $W_d(\mu, \nu) \preceq z = (0, \ldots,z_i, \ldots,0)$ 
\[
    W_d(\psi_{\#}(\mu), \psi_{\#}(\nu)) 
    \preceq (0, \ldots,(\rho z_i + s), \ldots,0),
\]
when we modify the $i$-th component.  
This inequality indeed holds due to the assumption on $\psi$: for every pair $(x, x')$, the condition ensures the bound componentwise. 
\end{proof}

Using the lemma above and decomposed Shift Reduction Lemma (\ref{lem:shift-reduction}), we can proceed block by block. By zeroing out the coordinates sequentially, we bound the \emph{increment} in R\'enyi divergence contributed by block $i$ by a quantity $\varepsilon^{\text{r\'enyi}}_i$. Summing these contributions yields an overall

\[
(\textstyle\sum_{i=1}^k \varepsilon^{\text{r\'enyi}}_i,\, q)\text{-RDP}.
\]

Since we already have an estimate for the number of steps required for unlearning in each component, summing them gives the bound on the total number of steps.

Applying the standard conversion from R\'enyi DP to $(\varepsilon,\delta)$-DP~\citep{Mironov_2017} gives

\begin{equation}
\varepsilon(\delta)
\;=\;
\sum_{i=1}^{k} \varepsilon^{\text{r\'enyi}}_i \;+\; \frac{\log(1/\delta)}{q-1}
\;=\;
\sum_{i=1}^{k} \varepsilon_i \;-\; (k-1)\frac{\log(1/\delta)}{q-1},
\end{equation}

where we define $\varepsilon_i := \varepsilon^{\text{r\'enyi}}_i + \tfrac{\log(1/\delta)}{q-1}$ for notational convenience.

\emph{We emphasize that we do not claim each block is itself an $(\varepsilon_i,\delta)$-mechanism; the equality above is an algebraic rewriting of the single $(\sum_i \varepsilon^{\text{r\'enyi}}_i,q)$-RDP bound after conversion.}

\paragraph{Equal blocks corollary.}

To guarantee the same budget as the baseline algorithm, it suffices to set 
\[
    \varepsilon^{\mathrm{renyi}}_i = \frac{\varepsilon^{\mathrm{renyi}}}{k}.
\]
Moreover, due to the randomness in the distribution of the model weights and gradients, the group norms are approximately equal. Hence we may treat 
\[
    \frac{1}{\sqrt{k}} C_0 \quad \text{and} \quad \frac{1}{\sqrt{k}} C_1
\]
as the individual clipping bounds for $B_i$.  

Keeping the same noise level (not necessarily minimal) does not change the number of steps $T$ for each group, since the factors of $\sqrt{k}$ compensate. Thus, the total cost is exactly $kT$ steps for the whole algorithm. Finally, we may choose a larger noise level, thereby reducing $T$. Since in each step we add substantially less noise to the model than in the baseline algorithm, we have the flexibility to increase the noise.

\section{Experiments} \label{app:experiments}

\subsection{Parameters and implementation details}

\paragraph{Compute resources.}
Experiments were run on NVIDIA A100 and NVIDIA T4 GPUs. The runtime-efficiency (RTE) values reported in the tables correspond to wall-clock time per unlearning run on the reported hardware.

\paragraph{Existing assets and licenses.}
We use standard public benchmark datasets, including MNIST~\citep{lecun1998mnist} and CIFAR-10~\citep{krizhevsky2009cifar10}, only for evaluation and do not redistribute them. MNIST is distributed under the Creative Commons Attribution-Share Alike 3.0 license; the original CIFAR-10 release does not specify a formal license and requests citation of the dataset creators. For reproduced empirical baselines, we use the official SaLUN implementation~\citep{fan2024salun} under its MIT license. When pretrained CIFAR models from \texttt{pytorch-cifar-models} are used, we follow its BSD-3-Clause license. All external assets are credited to their original authors and used according to their stated terms.

\label{model:linear-arch}
\begin{lstlisting}[language=Python, caption={Architecture of the model used for MNIST.}, label={app:linearnet}]
import torch.nn as nn
import torch.nn.functional as F

class LinearNet(nn.Module):
    def __init__(self, num_classes: int = 10):
        super().__init__()
        self.flatten = nn.Flatten()
        self.fc1 = nn.Linear(28 * 28, 2048)
        self.fc2 = nn.Linear(2048, 1024)
        self.fc3 = nn.Linear(1024, 512)
        self.fc4 = nn.Linear(512, 256)
        self.fc5 = nn.Linear(256, num_classes)

    def forward(self, x):
        x = self.flatten(x)
        x = F.relu(self.fc1(x))
        x = F.relu(self.fc2(x))
        x = F.relu(self.fc3(x))
        x = F.relu(self.fc4(x))
        x = self.fc5(x)
        return x
\end{lstlisting}

\label{app:hyperparams}

\begin{table}[H]
\centering
\footnotesize
\label{tab:mnist-train-hparams}
\caption{\textit{Training} hyperparameters for the fully trained MNIST model and the retrain baseline (identical settings).}
\fbox{
\begin{tabular}{@{}lcc@{}}
\textbf{Parameter} & \textbf{Fully trained} & \textbf{Retrain baseline} \\
\hline
Optimizer                  & SGD           & SGD \\
Learning rate              & 0.005          & 0.005 \\
Momentum                   & 0.9           & 0.9 \\
Weight decay               & $1\times10^{-5}$ & $1\times10^{-5}$ \\
Batch size                 & 64            & 64 \\
\end{tabular}
}
\end{table}

\begin{table}[H]
\centering
\footnotesize
\label{tab:vit-train-hparams}
\caption{Training hyperparameters for the ViT-Tiny model and the retrain baseline (identical settings). The model is instantiated via \texttt{timm.create\_model("vit\_tiny\_patch16\_224", pretrained=True)}}~\citep{rw2019timm}.
\fbox{
\begin{tabular}{@{}lcc@{}}
\textbf{Parameter} & \textbf{Pretrained ViT-Tiny} & \textbf{Retrain baseline} \\
\hline
Model architecture         & ViT-Tiny (patch 16) & same \\
Pretrained initialization  & Yes                 & Yes \\
Optimizer                  & AdamW               & AdamW \\
Learning rate              & $3\times 10^{-4}$   & $3\times 10^{-4}$ \\
Weight decay               & 0.05                & 0.05 \\
Batch size                 & 64                  & 64 \\
Epochs                     & 50                  & 50 \\
LR scheduler               & CosineAnnealingLR ($T_{\max}=50$) & same \\
Warmup                     & none                & none \\
Image resolution           & 224 (ViT input)     & 224 \\
Data augmentation          & Resize(224) + RandomHorizontalFlip + ToTensor + Normalize$(\mu,\sigma)$ & same \\
\end{tabular}
}
\end{table}

\begin{table}[H]
\centering
\footnotesize
\label{tab:cifar10-train-hparams}
\caption{Training hyperparameters for the fully trained CIFAR-10 model and the retrain baseline (identical settings).}
\fbox{
\begin{tabular}{@{}lcc@{}}
\textbf{Parameter} & \textbf{Fully trained} & \textbf{Retrain baseline} \\
\hline
Optimizer                  & SGD                & SGD \\
Learning rate              & 0.1                & 0.1 \\
Momentum                   & 0.9                & 0.9 \\
Weight decay               & $5\times10^{-4}$   & $5\times10^{-4}$ \\
Batch size                 & 256                & 256 \\
LR scheduler               & MultiStep(91, 136) $\times 0.1$; 182 epochs & same \\
Data augmentation          & Crop(32,pad=4)+Flip+Norm($\mu,\sigma$)      & same \\
\end{tabular}
}

\end{table}

\begin{table}[H]
\centering
\footnotesize
\label{tab:hparams-mnist-final}
\caption{Unlearning hyperparameters used in experiments on \textsc{MNIST}.}
\fbox{
\begin{tabular}{@{}lcc@{}}
\textbf{Parameter} & \textbf{Random 10\% deletion} & \textbf{Classwise deletion} \\
\hline
Initial distance bound $\Delta(\rho)$                & 0.5              & 3. \\
Gradient clipping $C_1$ & 5000.             & 900. \\
Unlearning step size $\eta$                  & $1\times10^{-4}$  & $1\times10^{-3}$ \\
Weight decay $\lambda$ (unlearning)          & 100                & 30 \\
      & 134 \\
Failure probability $\delta$                 & $10^{-5}$         & $10^{-4}$ \\
number of steps $T$ (per block)  & 2           & 2 \\
Fine-tuning learning rate                    & $1\times10^{-3}$  & $1\times10^{-3}$ \\
Fine-tuning weight decay                     & $1\times10^{-5}$  & $1\times10^{-5}$ \\
Fine-tuning momentum                         & 0.9               & 0.9 \\

\end{tabular}
}

\end{table}

\begin{table}[H]
\centering
\footnotesize
\label{tab:vit-blockwise-hparams}
\caption{Unlearning hyperparameters for Block-wise NFT and NFT used in experiments on ViT-Tiny.}
\fbox{
\begin{tabular}{@{}lcc@{}}

\textbf{Parameter} & \textbf{Random 10\% deletion} & \textbf{Classwise deletion} \\
\hline
Initial distance bound $\Delta(\rho)$      & 0.25               & 0.5 \\
Gradient clipping $C_1$             & 33                 & 33 \\
Unlearning step size $\eta$        & 0.002              & 0.002 \\
Weight decay $\lambda$ (unlearning)& 50                 & 50 \\
Failure probability $\delta$         & $10^{-5}$          & $10^{-5}$ \\
number of steps $T$ (per block)  & 2           & 2 \\
Fine-tuning optimizer                  & AdamW               & AdamW \\
Fine-tuning learning rate            & $5\times10^{-3}$   & $5\times10^{-3}$ \\
Fine-tuning weight decay             & 0                  & 0
\end{tabular}
}
\end{table}

\begin{table}[H]
\centering
\footnotesize
\label{tab:hparams-final}
\caption{Unlearning hyperparameters used in experiments on \textsc{CIFAR-10}.}
\fbox{
\begin{tabular}{@{}lcc@{}}

\textbf{Parameter} & \textbf{Random 10\% deletion} & \textbf{Classwise deletion} \\

\hline
Initial distance bound $\Delta(\rho)$                     & 50.            & 50.\\
Gradient clipping $C_1$             & 55                 & 55 \\
Unlearning step size $\eta$                       & $1\times10^{-4}$ & $1\times10^{-3}$ \\
Weight decay $\lambda$ (unlearning)               & 30               & 3 \\
Total privacy budget $\varepsilon$                & 165000                & 165000 \\
Failure probability $\delta$                      & $10^{-5}$        & $10^{-3}$ \\
number of steps $T$ (per block)  & 2           & 2 \\
Fine-tuning learning rate                         & $1\times10^{-3}$ & $1\times10^{-3}$ \\
Fine-tuning weight decay                          & $5\times10^{-4}$ & $5\times10^{-4}$ \\
Fine-tuning momentum                              & 0.9              & 0.9 \\

\end{tabular}
}
\end{table}

\subsection{Theory-motivated hyperparameter initialization}
\label{app:hyperparameter-selection}

For NFT-based methods, the certificate suggests a principled initialization for the noise scale and clipping radius. For fixed step budget $T$, learning rate $\gamma$, weight decay $\lambda$, proximity parameter $\Delta(\rho)$, and target $(\varepsilon,\delta)$, one can choose the R\'enyi parameters $(q,\varepsilon^{\mathrm{renyi}})$ to minimize the induced noise term under the conversion
\[
\varepsilon
=
\varepsilon^{\mathrm{renyi}}
+
\frac{\log(1/\delta)}{q-1}.
\]
This yields
\[
\sigma^2
=
\lambda\gamma(2-\lambda\gamma)
\frac{1 + (1-\lambda\gamma)^T}{1-(1-\lambda\gamma)^T}
\frac{8q}{\varepsilon^{\mathrm{renyi}}}
\Delta(\rho)^2,
\qquad
C_1
=
\frac{2\lambda\Delta(\rho)}{1-(1-\lambda\gamma)^T}.
\]
These expressions are not a universal rule, but a certificate-compatible warm start that can narrow the search space. 

\subsection{Definition of the MIA Metric Used in Table~\ref{tab:emp-comparison}}~\label{app:mia}

In Table~\ref{tab:emp-comparison} we report the MIA metric exactly as defined by \citet{jia2023model} in Appendix~C.3.

\paragraph{Definition.}
Following their protocol, an MIA model is first trained on
(a) a balanced subset of the retained dataset and
(b) a separate test dataset (disjoint from the forget set).
The trained MIA predictor is then evaluated on the forget set $D_f$
of the unlearned model $\theta_u$.

Let $\mathrm{TN}$ be the number of forgotten samples that the MIA
predictor classifies as \emph{non-members}.  
The reported metric is

\[
\text{MIA-Efficacy}
\;=\;
\frac{\mathrm{TN}}{|D_f|}.
\]

\paragraph{Interpretation.}
MIA-Efficacy measures the fraction of forgotten samples that the
attacker fails to recognize as training points.  
Thus, in this convention:
\begin{itemize}
    \item higher values indicate better unlearning quality;
    \item $\text{MIA-Efficacy} = 100\%$ means that all forgotten samples
    are predicted as non-members.
\end{itemize}

\subsection{Experiments with different block structure}

We further compare 4 constructions of the block mixing matrix $A$ (see Section~\ref{subsec:applications}):

\begin{itemize}
    \item a random $A$ with equal-sized blocks;
    \item a permutation-matrix $A$ with equal-sized blocks;
    \item a cyclic layer-wise grouping into $k$ blocks, assigning layer $\ell$ to block $\ell \bmod k$
    \item a two-block partition on head and the rest of the model (only for ViT-Tiny).
\end{itemize}

We provide results of experiments for both datasets (\textsc{MNIST} and \textsc{CIFAR-10}). For forget set $\mathcal{D}_f$, we use random 10\% deletion task.

\paragraph{Results on \textsc{MNIST}.}

Figure~\ref{fig:mnist_diff_types} shows \textsc{MNIST} trajectories at privacy budgets $\varepsilon\in\{30., 60.\}$, $\delta=10^{-5}$ and number of blocks $k = 2$. All three schemes exhibit very similar behavior during unlearning and subsequent fine-tuning; the main difference is the magnitude of the initial accuracy drop, which is slightly smaller for the odd-even split than for random splitting.

\begin{figure}[H]
    \centering
    \includegraphics[width=1.\linewidth]{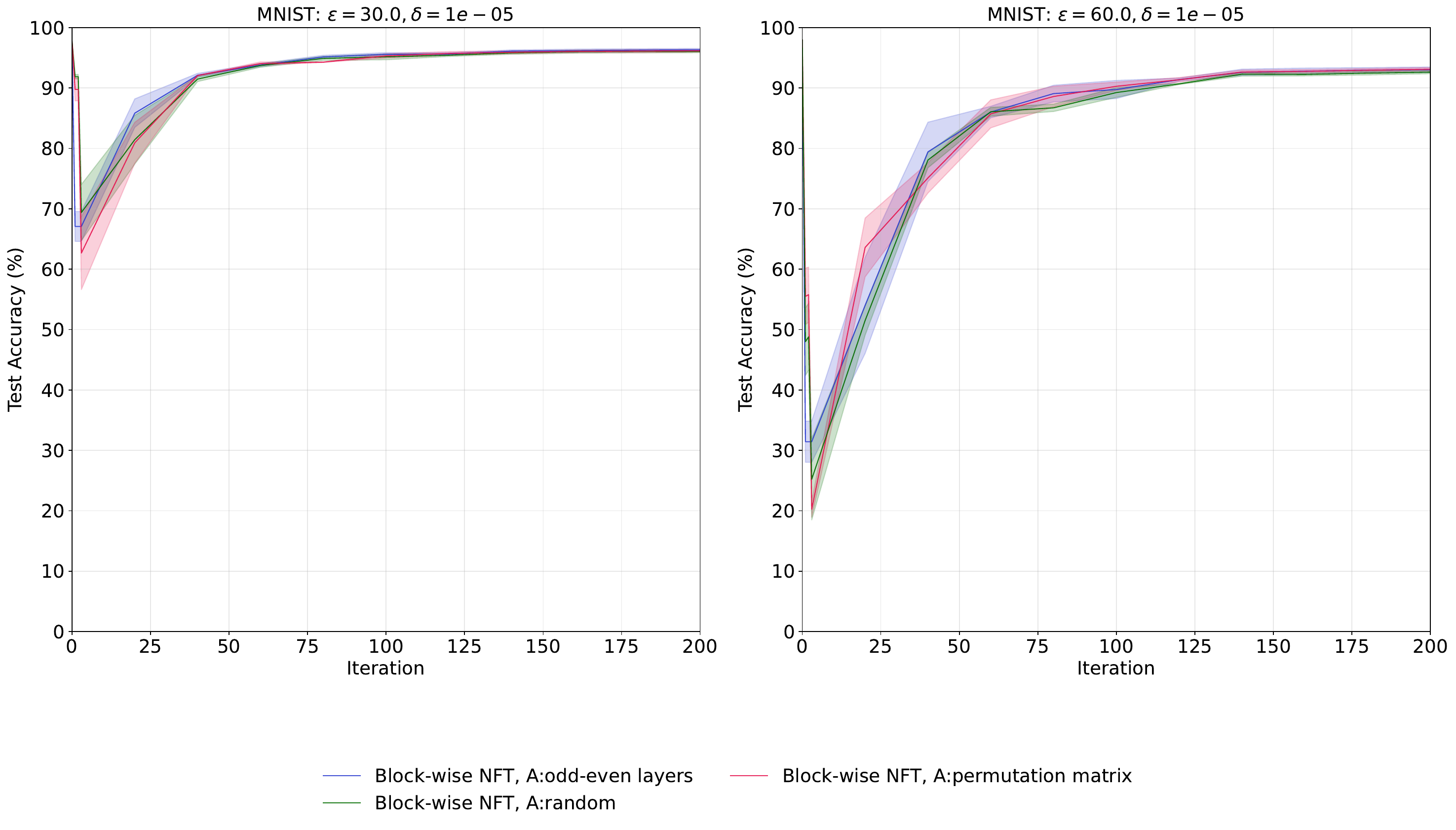}
    \caption{\textbf{Block-construction schemes} for Block-wise NFT on \textsc{MNIST} at $\varepsilon\in\{30, 60\}$ and $\delta=10^{-5}$.}
    \label{fig:mnist_diff_types}
\end{figure}

\paragraph{Results on \textsc{ViT-Tiny}.}

Figure~\ref{fig:vit-tiny_diff_types} shows \textsc{CIFAR-10} trajectories at privacy budgets $\varepsilon=50$, $\delta=10^{-5}$ and number of blocks $k = 4$. All schemes demonstrate comparable stability during unlearning and recovery under fine-tuning.

\begin{figure}[H]
    \centering
    \includegraphics[width=0.6\linewidth]{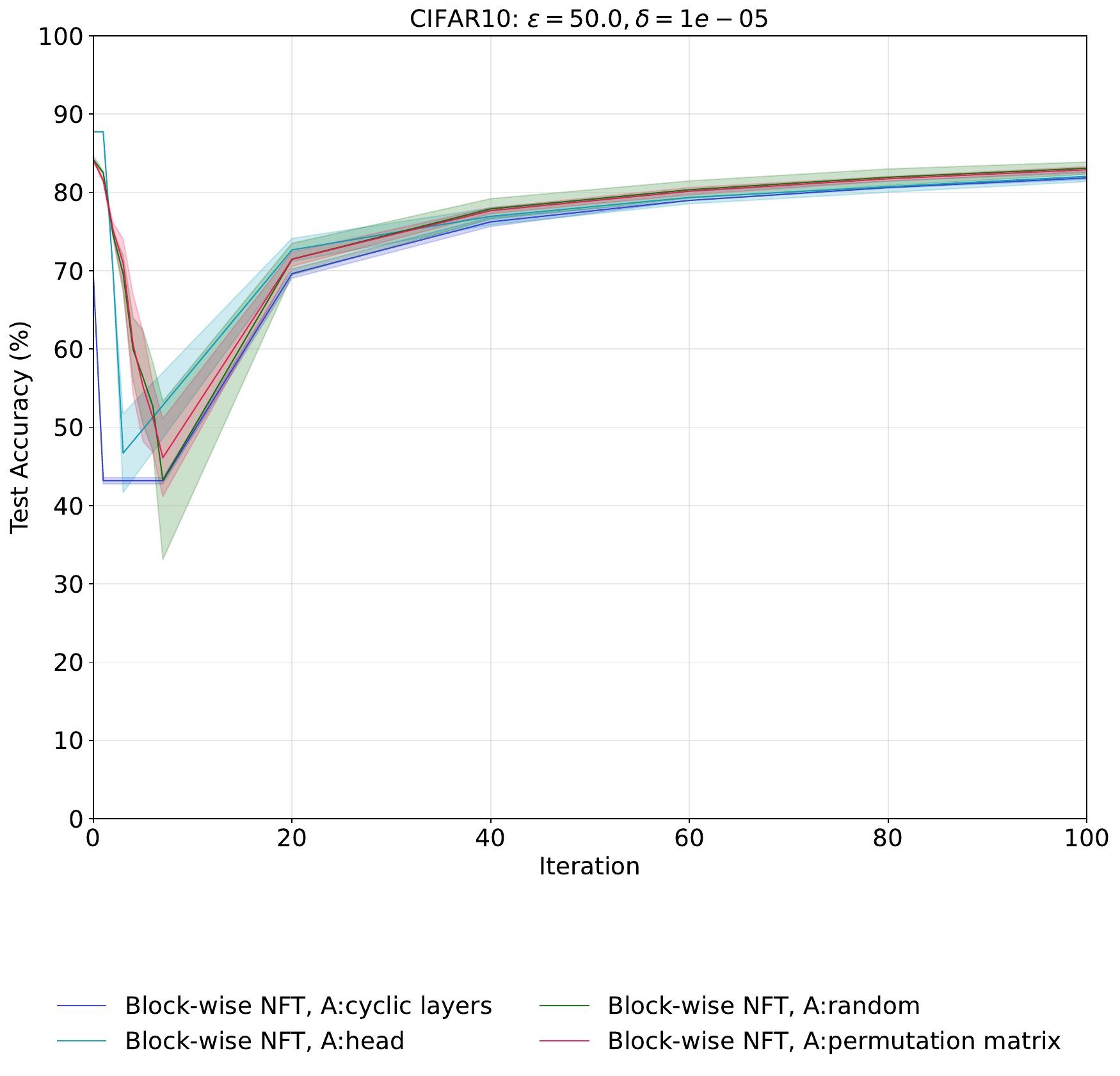}
    \caption{\textbf{Block-construction schemes} for Block-wise NFT on \textsc{CIFAR-10} for \textsc{ViT-Tiny} at $\varepsilon=50$ and $\delta=10^{-5}$.}
    \label{fig:vit-tiny_diff_types}
\end{figure}

\paragraph{Results on \textsc{ResNet-18}.}

Figure~\ref{fig:cifar_diff_types} shows \textsc{CIFAR-10} trajectories for ResNet-18 with privacy budget $\varepsilon = 165000.0$ and $\delta=10^{-5}$, with the number of blocks equals to 4 for the first 3 methods (partition head/body does not allow any other block number).

\begin{figure}[H]
    \centering
    \includegraphics[width=0.6\linewidth]{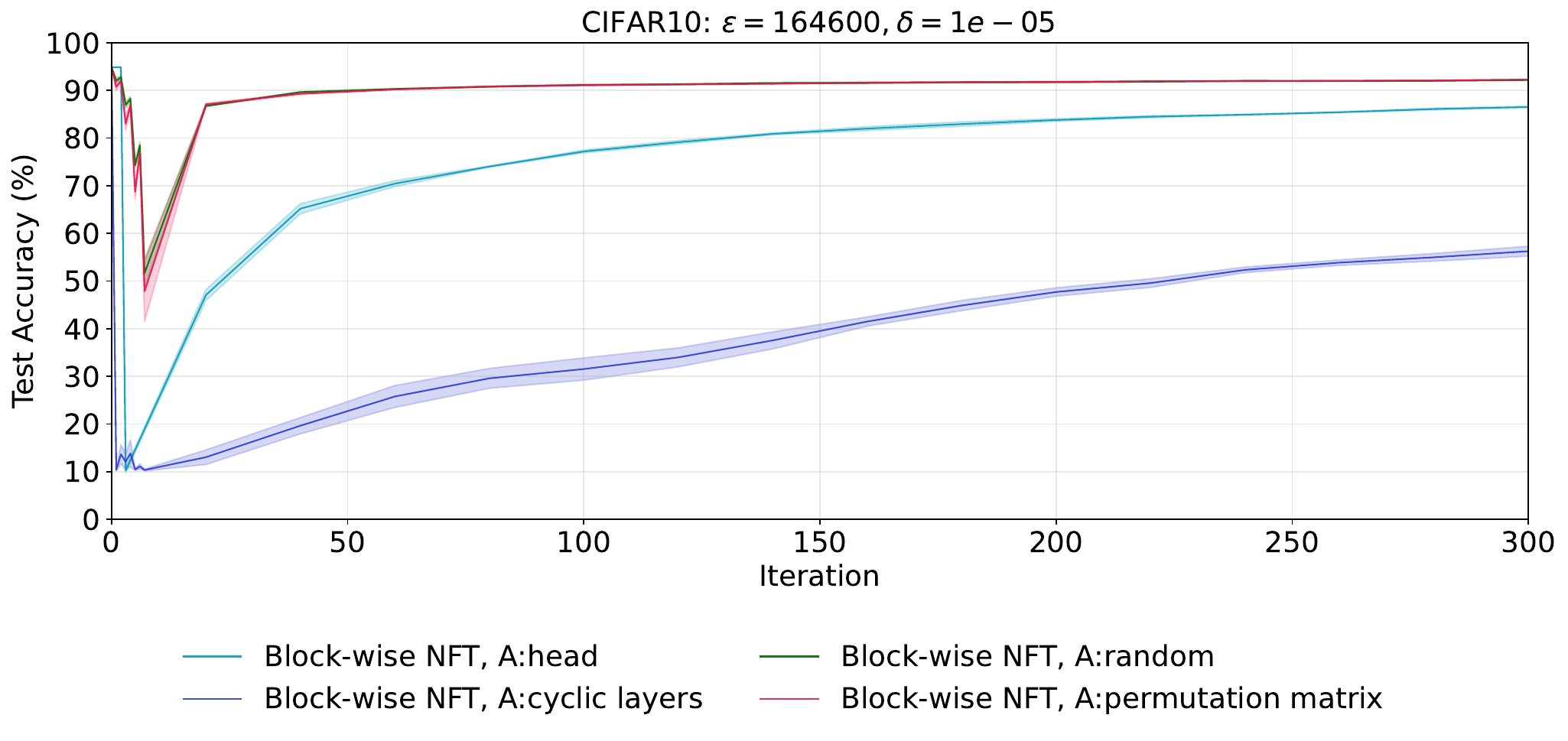}
    \caption{Block-construction schemes for Block-wise NFT on \textsc{CIFAR-10} at $\varepsilon=164600$ and $\delta=10^{-5}$.}
    \label{fig:cifar_diff_types}
\end{figure}

\subsection{Ablations}~\label{app:ablations}

\begin{table}[H]
\caption{Effect of the number of blocks $k$ on ViT-Tiny/CIFAR-10 for privacy budget $\varepsilon=50$. Min Acc. = minimum accuracy during unlearning before post-fine-tuning; Final Acc. = accuracy after post-fine-tuning; RTE is reported in \textit{seconds}; RTE (blocks only) excludes post-fine-tuning.}
\label{tab:k_ablation}
\centering
\small
\begin{tabular}{cccccc}
\hline
$k$ & Min Acc. & Final Acc. & RTE & RTE (blocks only) & $\varepsilon$ \\
\hline
2  & 37.6$\pm$7.71 & 89.69$\pm$0.49 & 90.77$\pm$0.93  & 1.46$\pm$0.3  & 50 \\
4  & 41.69$\pm$6.06 & 89.89$\pm$0.56 & 91.37$\pm$0.51  & 2.33$\pm$0.34 & 50 \\
8  & 40.12$\pm$1.09 & 89.74$\pm$0.2  & 92.01$\pm$0.67  & 3.97$\pm$0.07 & 50 \\
16 & 34.25$\pm$8.2  & 89.58$\pm$0.66 & 95.27$\pm$0.55  & 9.01$\pm$0.21 & 50 \\
32 & 34.38$\pm$10.11 & 89.62$\pm$0.17 & 108.18$\pm$0.32 & 24.66$\pm$0.11 & 50 \\
\hline
\end{tabular}
\end{table}

\begin{table}[H]
\caption{Effect of the number of blocks $k$ on ViT-Tiny/CIFAR-10 for privacy budget $\varepsilon=30$. Min Acc. = minimum accuracy during unlearning before post-fine-tuning; Final Acc. = accuracy after post-fine-tuning; RTE is reported in \textit{seconds}; RTE (blocks only) excludes post-fine-tuning.}
\label{tab:k_ablation_eps30}
\centering
\small
\begin{tabular}{cccccc}
\hline
$k$ & Min Acc. & Final Acc. & RTE & RTE (blocks only) & $\varepsilon$ \\
\hline
2  & 21.77$\pm$3.93 & 88.45$\pm$1.1  & 91.9$\pm$0.51   & 1.45$\pm$0.3  & 30 \\
4  & 20.62$\pm$5.63 & 89.27$\pm$0.68 & 91.4$\pm$0.58   & 2.03$\pm$0.06 & 30 \\
8  & 24.34$\pm$9.09 & 89.02$\pm$0.33 & 92.81$\pm$0.32  & 3.97$\pm$0.11 & 30 \\
16 & 13.94$\pm$4.01 & 88.52$\pm$1.29 & 95.87$\pm$0.76  & 9.01$\pm$0.08 & 30 \\
32 & 13.47$\pm$0.43 & 88.74$\pm$0.55 & 108.65$\pm$0.33 & 24.59$\pm$0.03 & 30 \\
\hline
\end{tabular}
\end{table}

\begin{table}[H]
\centering
\small
\caption{Effect of block construction and the number of blocks $k$ on ViT-Tiny/CIFAR-10 at privacy budget $\varepsilon=50$. Min Acc. = minimum accuracy during unlearning before post-fine-tuning; Final Acc. = accuracy after post-fine-tuning; RTE is reported in \textit{seconds}; RTE (blocks only) excludes post-fine-tuning.}
\label{tab:block_mode_ablation}
\begin{tabular}{cccccc}
\hline
$k$ & Block mode & Min Acc. & Final Acc. & RTE & RTE (blocks only) \\
\hline
2 & random & 36.39$\pm$15.99 & 89.81$\pm$0.52 & 92.37$\pm$0.32 & 1.28$\pm$0.02 \\
2 & random permutation & 39.74$\pm$3.3 & 89.53$\pm$0.29 & 90.98$\pm$0.75 & 0.7$\pm$0.03 \\
2 & layer-wise & 38.94$\pm$1.06 & 89.9$\pm$0.22 & 91.25$\pm$0.69 & 0.8$\pm$0.29 \\
\hline
4 & random & 41.69$\pm$6.06 & 89.89$\pm$0.56 & 92.38$\pm$0.18 & 2.15$\pm$0.13 \\
4 & random permutation & 36.35$\pm$1.46 & 89.67$\pm$0.33 & 91.37$\pm$0.58 & 1.39$\pm$0.02 \\
4 & layer-wise & 35.54$\pm$5.11 & 89.73$\pm$0.16 & 91.52$\pm$0.32 & 1.3$\pm$0.01 \\
\hline
8 & random & 40.12$\pm$1.09 & 89.74$\pm$0.2 & 93.42$\pm$0.14 & 3.95$\pm$0.06 \\
8 & random permutation & 44.41$\pm$2.19 & 89.98$\pm$0.19 & 93.03$\pm$0.19 & 3.3$\pm$0.19 \\
8 & layer-wise & 34.34$\pm$6.56 & 89.73$\pm$0.12 & 92.05$\pm$0.34 & 2.89$\pm$0.02 \\
\hline
\end{tabular}
\end{table}

\begin{table}[H]
\caption{Sensitivity to randomness in block partitioning for $k=2$ on ViT-Tiny/CIFAR-10, reported for different privacy budgets $\varepsilon$. Min Acc. = minimum accuracy during unlearning before post-fine-tuning; Final Acc. = accuracy after post-fine-tuning.}
\label{tab:block_partition_randomness}
\centering
\small
\begin{tabular}{cccc}
\hline
$k$ & $\varepsilon$ & Min Acc. & Final Acc. \\
\hline
2 & 50 & 38.92$\pm$7.1  & 89.87$\pm$0.49 \\
2 & 70 & 51.58$\pm$2.37 & 90.09$\pm$0.41 \\
2 & 90 & 57.56$\pm$0.0  & 90.52$\pm$0.0 \\
\hline
\end{tabular}
\end{table}

\subsection{Class-wise experiments}

We evaluate class-wise forgetting on \textsc{MNIST}, where we forget class 5 entirely. 

Table~\ref{tab:comparison-classwise} reports UA/RA/TA/MIA for $\varepsilon \in \{75, 90\}$: Block-wise NFT attains UA=100 and MIA=100 at both budgets while preserving higher RA/TA than NFT.

Figure~\ref{fig:classwise-mnist} shows the corresponding learning dynamics, with a smaller transient drop and clearer recovery under fine-tuning.

\begin{figure}[H]
    \centering
    \includegraphics[width=1.\linewidth]{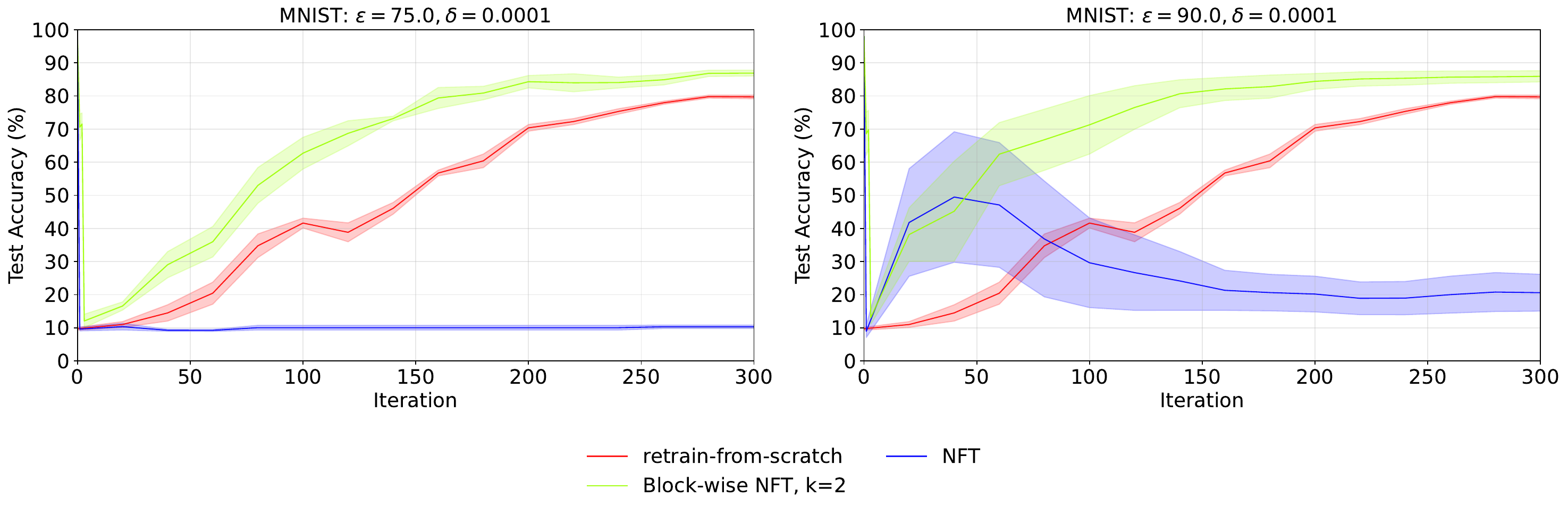}
    \caption{\textsc{MNIST}, class-wise forgetting. Evaluating unlearning metrics for NFT and Block-wise NFT across $\varepsilon\in\{75, 90\}.$ for unlearning class 5.}
    \label{fig:classwise-mnist}
\end{figure}

\begin{table}[H]
\caption{\textbf{\textsc{MNIST}, class-wise forgetting.} Evaluating unlearning metrics for NFT and Block-wise NFT for $\varepsilon\in\{75, 90\}$. Block-wise NFT attains the best possible score for UA and MIA while preserving RA/TA better than NFT and close to the retrained model.}
\label{tab:comparison-classwise}
\begin{center}
\begin{tabular}{lccccc}
\multicolumn{1}{c}{\bf Method} & \multicolumn{1}{c}{\bf UA} & \multicolumn{1}{c}{\bf RA} & \multicolumn{1}{c}{\bf TA} & \multicolumn{1}{c}{\bf MIA}
\\ \hline \\
Retrain & 100.00 & 99.83 & 89.33 & 100.00\\ 
\hline \\
NFT (75) & 100.00 & 15.76 & 14.22 & 100.00 & \\ 
Block-wise NFT (75) & 100 & 97.47 & 88.06 & 100.00 & \\
NFT (90) & 99.96 & 38.16 & 34.82 & 100.00 & \\ 
Block-wise NFT (90) & 100 & 98.09 & 88.53 & 100.00 & \\ 
\end{tabular}
\end{center}

\end{table}

\subsection{Additional experiments}

We also report results for a more challenging setup in which 50\% of the training data are selected uniformly at random for removal.

\begin{figure}[H]
    \centering
    \includegraphics[width=1.\linewidth]{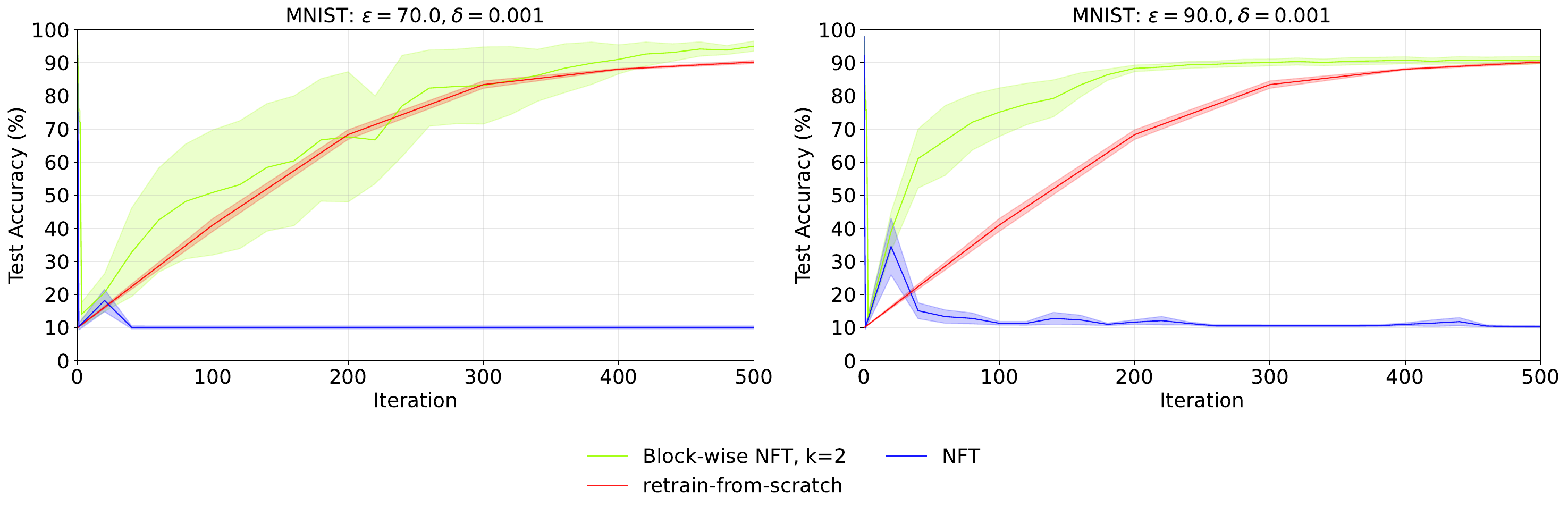}
    \caption{\textbf{Random 50\% deletion.} We compare standard NFT with Block-wise NFT with two blocks, and show the final accuracy of retraining from scratch for reference. Despite the larger forget set, Block-wise NFT achieves accuracy recovery.}
    \label{fig:placeholder}
\end{figure}

\subsection{Extended experiments for CIFAR-10}
\label{app:classwise-exps}

In this section we provide the full per-class unlearning results for
\textsc{CIFAR-10}.  
For each of the 10 classes, we delete the class entirely, retrain the baseline model from scratch on the retained data, and apply Block-wise NFT under the same setting. The results are presented in Table~\ref{tab:cifar10-perclass}.

\paragraph{Note on runtime.}
The RTE values reported in Table~\ref{tab:emp-comparison} were obtained in a different runtime session than the experiments in the main paper, resulting in a slight systematic shift in wall-clock time. Since RTE is used only as a relative measure within each experiment, this does not affect any comparisons or conclusions.

\begin{table}[H]
\centering
\caption{\textbf{Per-class deletion results on \textsc{CIFAR-10}.}
For each deleted class, we report results for retraining-from-scratch
and for Block-wise NFT (our method). Metrics: unlearned accuracy (UA),
retain accuracy (RA), test accuracy (TA), membership-inference score (MIA),
and relative training effort (RTE, minutes).}
\label{tab:cifar10-perclass}
\scalebox{0.9}{
\begin{tabular}{ccccccccccc}
\toprule
& \multicolumn{5}{c}{\bf Retrain} 
& \multicolumn{5}{c}{\bf Block-wise NFT} \\
\bf Class 
& UA & RA & TA & MIA & RTE 
& UA & RA & TA & MIA & RTE \\
\cmidrule(r){2-6} \cmidrule(l){7-11}
\textbf{0} 
& 100.00 & 100.00 & 85.21 & 100.00 & 44.35 
& 100.00 & 97.17 & 83.57 & 100.00 & 0.80 \\
\textbf{1} 
& 100.00 & 100.00 & 84.98 & 100.00 & 44.29
& 100.00 & 96.89 & 82.93 & 100.00 & 0.81 \\
\textbf{2} 
& 100.00 & 100.00 & 85.56 & 100.00 & 44.29
& 100.00 & 97.24 & 83.64 & 100.00 & 0.80 \\
\textbf{3} 
& 100.00 & 100.00 & 86.47 & 100.00 & 44.26 
& 100.00 & 97.72 & 84.99 & 100.00 & 0.81 \\
\textbf{4} 
& 100.00 & 100.00 & 85.00 & 100.00 & 44.53
& 100.00 & 97.06 & 83.73 & 100.00 & 0.81 \\
\textbf{5} 
& 100.00 & 100.00 & 86.14 & 100.00 & 44.35
& 100.00 & 96.18 & 84.33 & 100.00 & 0.80 \\
\textbf{6} 
& 100.00 & 100.00 & 85.06 & 100.00 & 44.35 
& 100.00 & 97.05 & 82.87 & 100.00 & 0.81 \\
\textbf{7} 
& 100.00 & 100.00 & 84.83 & 100.00 & 44.07 
& 100.00 & 96.96 & 83.32 & 100.00 & 0.80 \\
\textbf{8} 
& 100.00 & 100.00 & 85.03 & 100.00 & 44.47 
& 100.00 & 96.97 & 83.21 & 100.00 & 0.80 \\
\textbf{9} 
& 100.00 & 100.00 & 84.9 & 100.00 & 44.29 
& 100.00 & 97.10 & 83.32 & 100.00 & 0.81 \\
\cmidrule(r){2-6} \cmidrule(l){7-11}
\textbf{Mean} 
& 100.00 & 100.00 & 85.32 & 100.00 & 44.33 
& 100.00 & 96.88 & 83.79 & 100.00 & 0.80 \\
\textbf{Std} 
& 0 & 0 & 0.53 & 0 & 0.11 
& 0 & 0.47 & 0.63 & 0 & 0.01 \\
\bottomrule
\end{tabular}
}
\end{table}

\section{Recalculation of $\Delta(\rho)$}~\label{app:budget-recalculation}

For each block, the NFT algorithm is parameterized by
$\lambda, \gamma, \Delta(\rho), C_1$ and the budget parameters
$\varepsilon_{\text{r\'enyi}}$ and $q$.

Given these parameters, the noise variance $\sigma^2$ and the number of unlearning steps $T$ are determined by the algorithm \ref{alg:block-noisy-finetuning}.

Importantly, only the parameters $(\lambda, \gamma, C_1, \sigma^2, T)$ affect the optimization trajectory of the algorithm. In particular, if two NFT instantiations share the same values of $(\lambda, \gamma, C_1, \sigma^2, T)$, then they generate identical parameter updates (with shared random seed) and therefore follow the same optimization path.

This observation allows us to state the following result.

\begin{proposition}
    Consider an NFT algorithm with parameters $\lambda, \gamma, \Delta(\rho), C_1$ and budget parameter $B := q / \varepsilon_{\text{r\'enyi}}$, with noise variance $\sigma^2$ and number of unlearning steps $T$. Define
    \[
        \Delta^{\text{new}}(\rho) := m \cdot \Delta(\rho),
        \qquad
        B^{\text{new}} := \frac{1}{m^2} \cdot B ,
    \]
    for some $m \ge 1$.

    Then if we run the algorithm \ref{alg:block-noisy-finetuning} with parameters $\lambda, \gamma, \Delta^{new}(\rho), C_1, B^{new}$, \emph{while keeping the same values of $\sigma^2$ and $T$}, then the resulting algorithm follows exactly the same optimization trajectory and is certified with a \emph{different budget}, correspondent to $B^{new}$.
\end{proposition}

\begin{proof}
In Algorithm~\ref{alg:block-noisy-finetuning}, the noise variance and the number of unlearning steps are computed according to
    \[
        \sigma^2 =
        \gamma(2 - \gamma \lambda)  \tfrac{q} {\varepsilon^{\text{r\'enyi}}} 
        \Bigl(2 - \tfrac{\lambda \Delta(\rho)}{2C_1}\Bigr) \Delta(\rho)C_1,
    \]
and
    \[
      T(\sigma^2_{\min})
      \;=\;
      \frac{\log\!\bigl(1 - \tfrac{\lambda \Delta(\rho)}{2C_1}\bigr)}{\log(1-\gamma\lambda)}.
    \]

Under the transformation $\Delta(\rho) \to \Delta^{\text{new}}(\rho)$, $C_1 \to m \cdot C_1$, and $B \to B^{\text{new}}$, the values of $\sigma^2$ and $T$ remain unchanged.

The resulting algorithm is therefore parameterized by
$(\lambda, \gamma, m C_1, \sigma^2, T)$ and admits a valid certification. In the certification proof \citep{koloskova2025certified}, the only property required of the gradient clipping constant is that the norm of the clipped gradient does not exceed this constant. For this reason, for the algorithm with the stricter clipping threshold $C_1$, the certification may safely use the larger value $m C_1$.

Consequently, an algorithm parameterized by $(\lambda, \gamma, C_1, \sigma^2, T)$ is certified with budget correspondent to $B^{\text{new}}$, and its optimization trajectory -- fully determined by $(\lambda, \gamma, C_1, \sigma^2, T)$ -- coincides exactly with the original trajectory.

\end{proof}

What does this mean in practice? Throughout our discussion, certification has always been understood as being conditional on the chosen value of $\Delta(\rho)$.

The proposition above shows that even if the selected $\Delta(\rho)$ turns out to be inaccurate, this does not invalidate the certification of the algorithm. The algorithm itself remains certified; however, the resulting privacy budget may differ from the one originally anticipated. In such a case, the budget can be recalculated using the updated estimate of $\Delta(\rho)$.

Importantly, this recalculation does not alter the optimization trajectory of the algorithm. The sequence of updates remains unchanged. Consequently, the certification budget can be viewed as a function of the chosen estimate of $\Delta(\rho)$, while the underlying algorithmic behavior stays fixed.

\subsection{How to calculate a new budget}

To decrease $\dfrac{1}{B}$ by $m^2$ times, we need to understand connection between $\dfrac{1}{B}$ and unlearning budget $(\varepsilon, \delta)$. 

For given $(\varepsilon, \delta)$, the pair $\varepsilon_{\text{r\'enyi}}, q$ is calculated to minimize $\sigma^2$, or, equally, $B$.
\[
B = \dfrac{q}{\varepsilon_{\text{r\'enyi}}} \to \text{min}, \;\;\; \text{where} \;\;\; 
\varepsilon  = \varepsilon^{\text{r\'enyi}} + \tfrac{\log(1/\delta)}{q-1} 
\]
Denote $a := \log(1/\delta)$. From the constraint,
\[
\varepsilon_{\text{r\'enyi}}(q)=\varepsilon-\frac{a}{q-1},
\qquad
B(q)=\frac{q}{\varepsilon-\frac{a}{q-1}},
\qquad
q>1+\frac{a}{\varepsilon}.
\]
Let $t=q-1$. Minimizing $B$ is equivalent to minimizing
\[
B(t)=\frac{t(t+1)}{\varepsilon t-a},\qquad t>\frac{a}{\varepsilon}.
\]
The unique minimizer satisfies
\[
\varepsilon t^2-2at-a=0
\quad\Rightarrow\quad
t^\star=\frac{a+\sqrt{a(a+\varepsilon)}}{\varepsilon}.
\]
Hence
\[
q^\star
=1+t^\star
=1+\frac{a+\sqrt{a(a+\varepsilon)}}{\varepsilon},
\]
and
\[
\varepsilon_{\text{r\'enyi}}^\star
=\varepsilon-\frac{a}{t^\star}
=\varepsilon+a-\sqrt{a(a+\varepsilon)}.
\]
Therefore,
\[
q^\star = 1+\frac{\log(1/\delta)+\sqrt{\log(1/\delta)\bigl(\log(1/\delta)+\varepsilon\bigr)}}{\varepsilon},
\qquad
\varepsilon_{\text{r\'enyi}}^\star
=\varepsilon+\log(1/\delta)-\sqrt{\log(1/\delta)\bigl(\log(1/\delta)+\varepsilon\bigr)}.
\]

The graph \ref{fig:B-and-varepsilon-connection-1000} illustrates dependency $\varepsilon$ and $1/B$ with fixed $\delta = 10^{-5}$.

\begin{figure}[H]
    \centering
    \includegraphics[width=1.\linewidth]{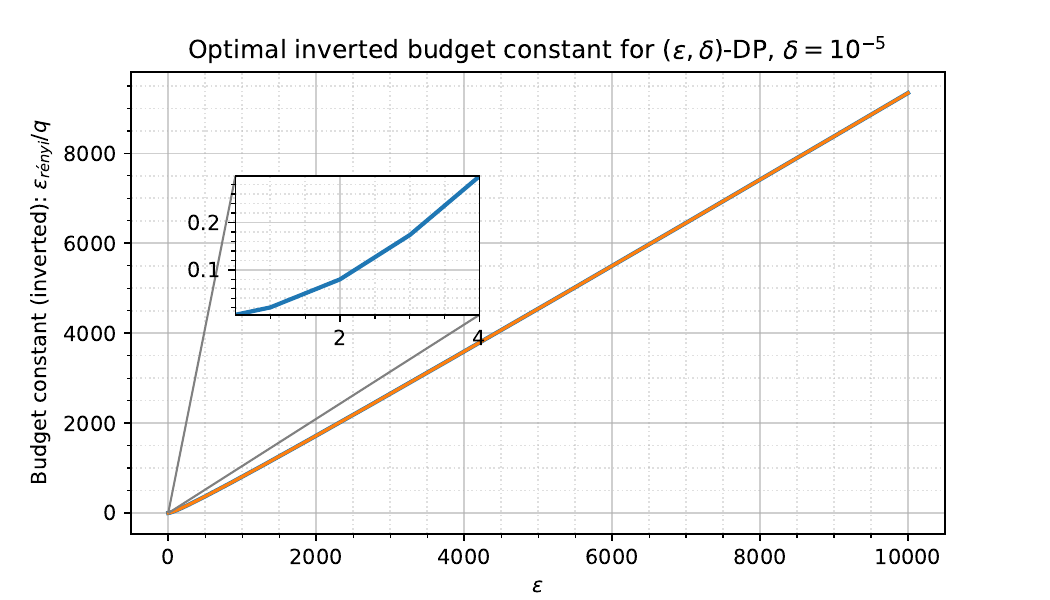}
    \caption{Budget constant connection to budget $\varepsilon$}
    \label{fig:B-and-varepsilon-connection-1000}
\end{figure}

\section{Delta estimation}~\label{app:delta-estimation}

\paragraph{Empirical sensitivity estimation for $\Delta(\rho)$.}
To instantiate $\Delta(\rho)$ in our experiments, we perform an architecture-level sensitivity estimation procedure \emph{prior} to receiving any unlearning request. The goal is to quantify how far the trained parameters can move under small, controlled changes of the training set, under a coupled comparison of training trajectories.

\paragraph{Coupled training protocol.}
We run training multiple times (10 runs in our experiments) under fixed training randomness: we fix all seeds and the data order, and keep the optimization hyperparameters unchanged. We then construct a new training stream by applying a controlled dataset modification while preserving the same mini-batch structure as much as possible. Concretely, for \emph{random $10\%$ deletion} we remove a randomly selected $10\%$ subset of training examples; for \emph{classwise deletion} we remove all training examples belonging to a randomly selected target class. The removed samples are replaced by retained samples so that the overall training pipeline remains coupled across the two runs.

\paragraph{Distance measurement and aggregation.}
For each run, we train on the original dataset and on its perturbed counterpart, and measure the resulting parameter distance. We aggregate these distances across the 10 runs and use the resulting value as an empirical upper estimate for $\Delta(\rho)$ for the considered architecture and deletion setting. 

As illustrated by the calibration plots, the \emph{time evolution} of the proximity -- i.e., how the distance between coupled trajectories grows and stabilizes over training -- is often consistent across runs under fixed coupling. This suggests that, for a given architecture and training pipeline, one can characterize this trajectory in advance and use it as a practical guide (e.g., for selecting conservative instantiations of $\Delta(\rho)$), while trading a fully distribution-free bound for a fast, architecture-specific calibration.

In addition, storing intermediate checkpoints during training can be beneficial: in practice, later checkpoints with comparable accuracy may exhibit smaller estimated proximity, and initializing unlearning from such a checkpoint can reduce the resulting certified budget for the same unlearning trajectory.

Figures~\ref{fig:delta-traj-mnist-10p-lr5e-3}--\ref{fig:delta-traj-cifar-rn18-class-lr1e-1} visualize the parameter distance $\Delta$ (right axis) together with test accuracy (left axis) over training steps for different deletion types and architectures. Here $\Delta:=\Delta(1)$ denotes the $\ell_2$ distance between the parameters of the two coupled runs. Notably, these trajectories often exhibit a regime where accuracy is already high while $\Delta$ is still relatively small, motivating the use of earlier checkpoints when a tighter proximity (and thus a smaller certified budget) is desired. On MNIST, varying the learning rate can noticeably change the magnitude of $\Delta$ even when the final accuracy is similar (Figures~\ref{fig:delta-traj-mnist-10p-lr5e-3} and~\ref{fig:delta-traj-mnist-10p-lr1e-2}).

\begin{figure}[H]
    \centering
    \includegraphics[width=0.8\linewidth]{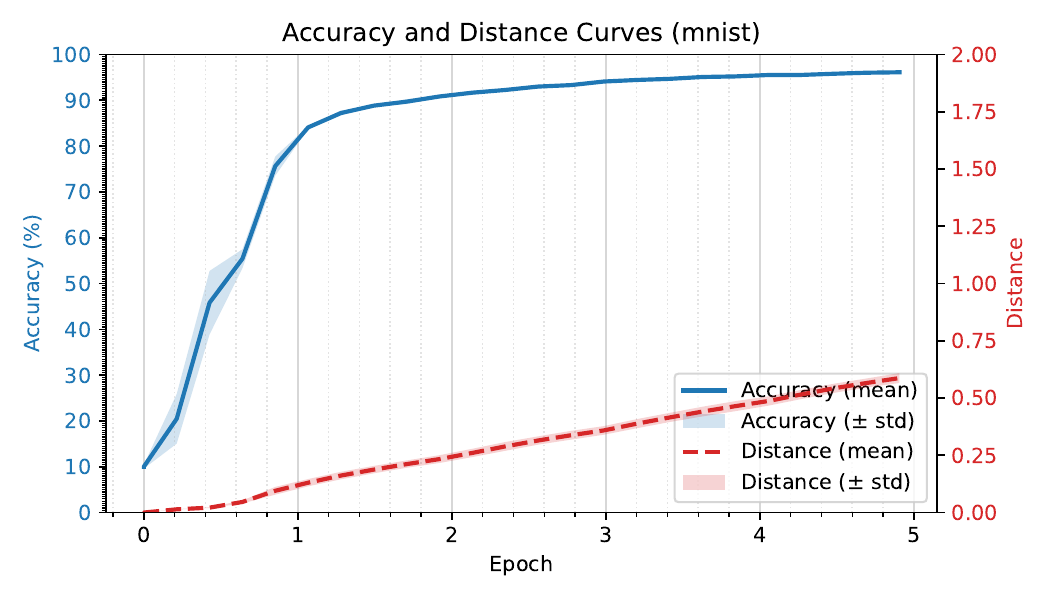}
    \caption{\textbf{MNIST (4-layer MLP), random 10\% deletion, lr=0.005.}
    Test accuracy (left) and parameter distance $\Delta$ (right) over training steps.}
    \label{fig:delta-traj-mnist-10p-lr5e-3}
\end{figure}

\begin{figure}[H]
    \centering
    \includegraphics[width=0.8\linewidth]{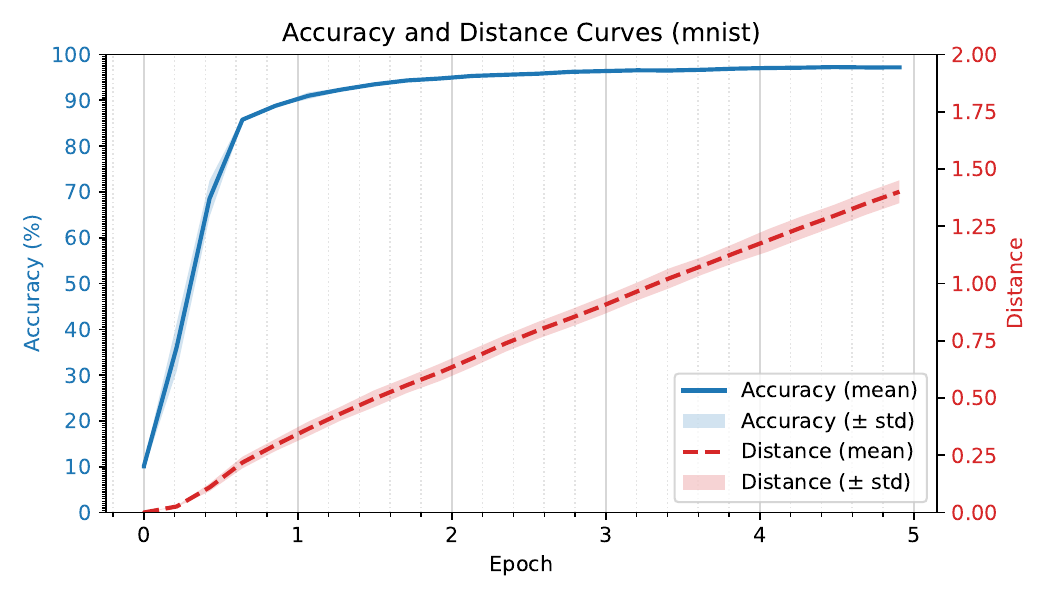}
    \caption{\textbf{MNIST (4-layer MLP), random 10\% deletion, lr=0.01.}
    Test accuracy (left) and parameter distance $\Delta$ (right) over training steps.}
    \label{fig:delta-traj-mnist-10p-lr1e-2}
\end{figure}

\begin{figure}[H]
    \centering
    \includegraphics[width=0.8\linewidth]{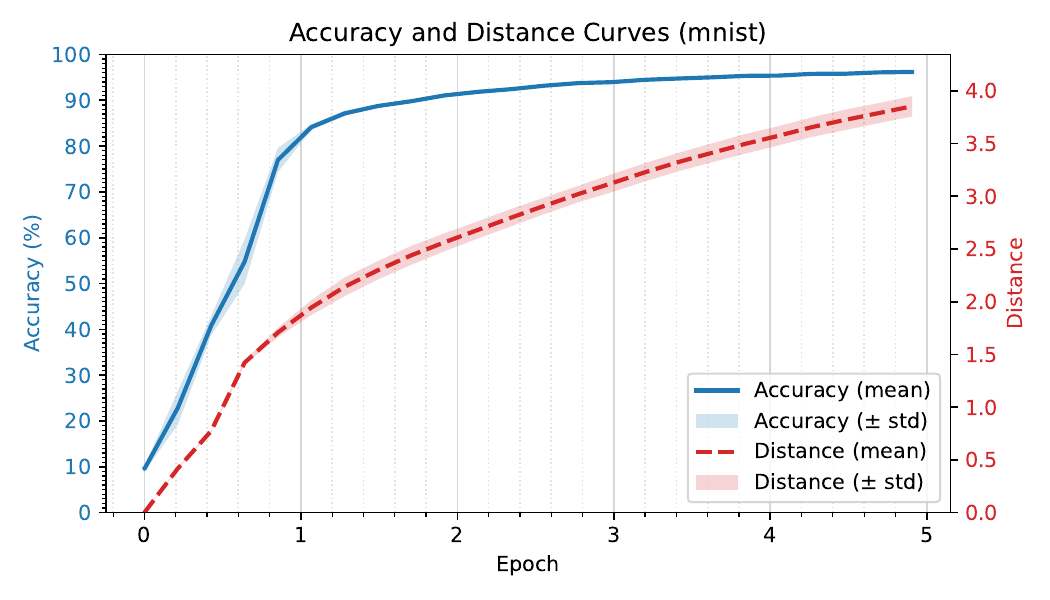}
    \caption{\textbf{MNIST (4-layer MLP), class deletion, lr=0.005.}
    Test accuracy (left) and parameter distance $\Delta$ (right) over training steps.}
    \label{fig:delta-traj-mnist-class-lr5e-3}
\end{figure}

\begin{figure}[H]
    \centering
    \includegraphics[width=0.8\linewidth]{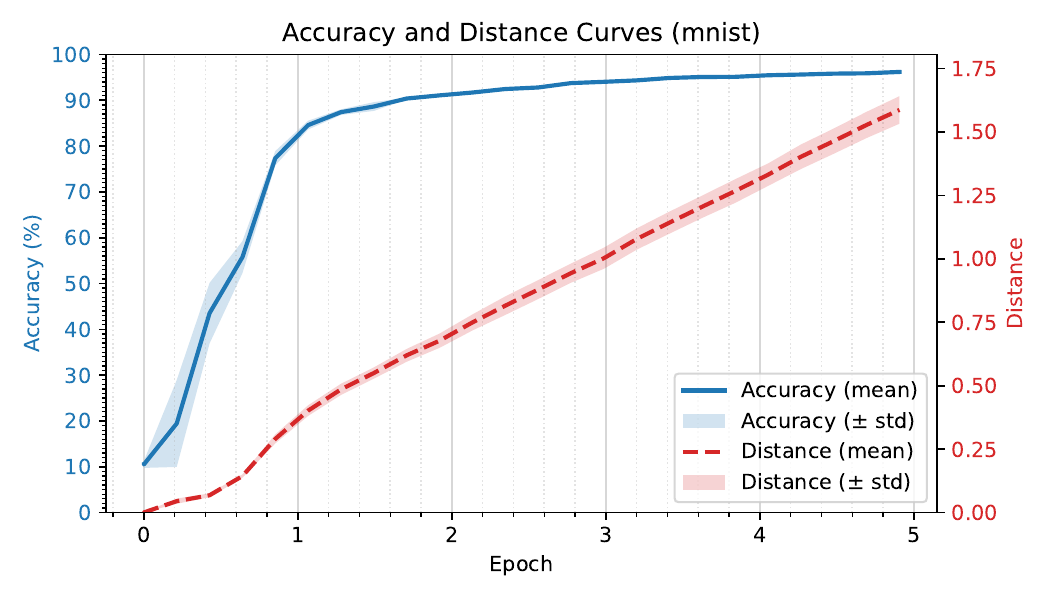}
    \caption{\textbf{MNIST (4-layer MLP), random 50\% deletion, lr=0.005.}
    Test accuracy (left) and parameter distance $\Delta$ (right) over training steps.}
    \label{fig:delta-traj-mnist-50p-lr5e-3}
\end{figure}

\begin{figure}[H]
    \centering
    \includegraphics[width=0.8\linewidth]{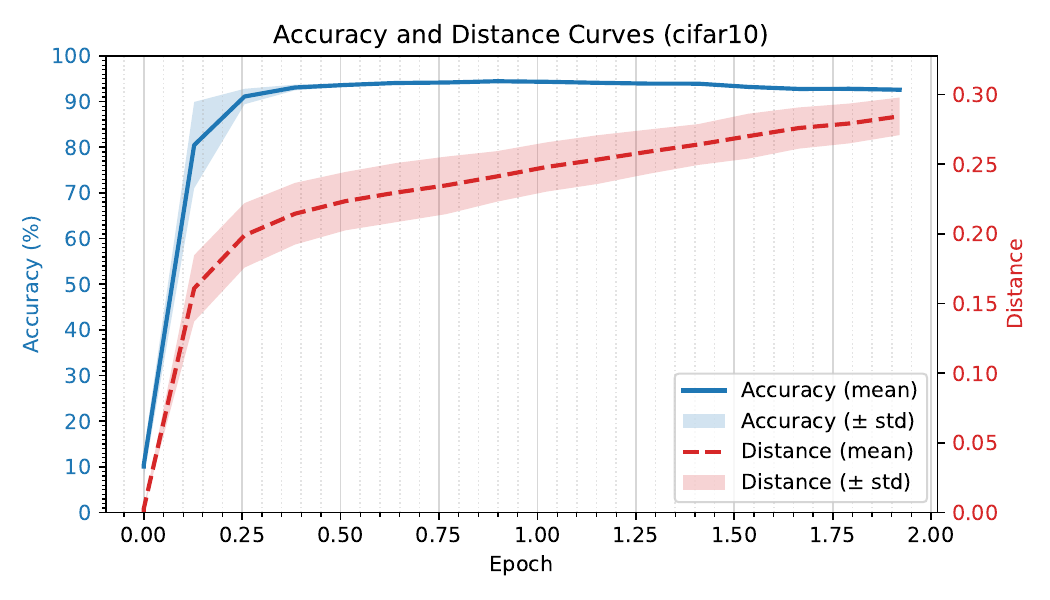}
    \caption{\textbf{CIFAR-10 (ViT-Tiny), random 10\% deletion, lr=3e-4.}
    Test accuracy (left) and parameter distance $\Delta$ (right) over training steps.}
    \label{fig:delta-traj-cifar-vit-10p-lr3e-4}
\end{figure}

\begin{figure}[H]
    \centering
    \includegraphics[width=0.8\linewidth]{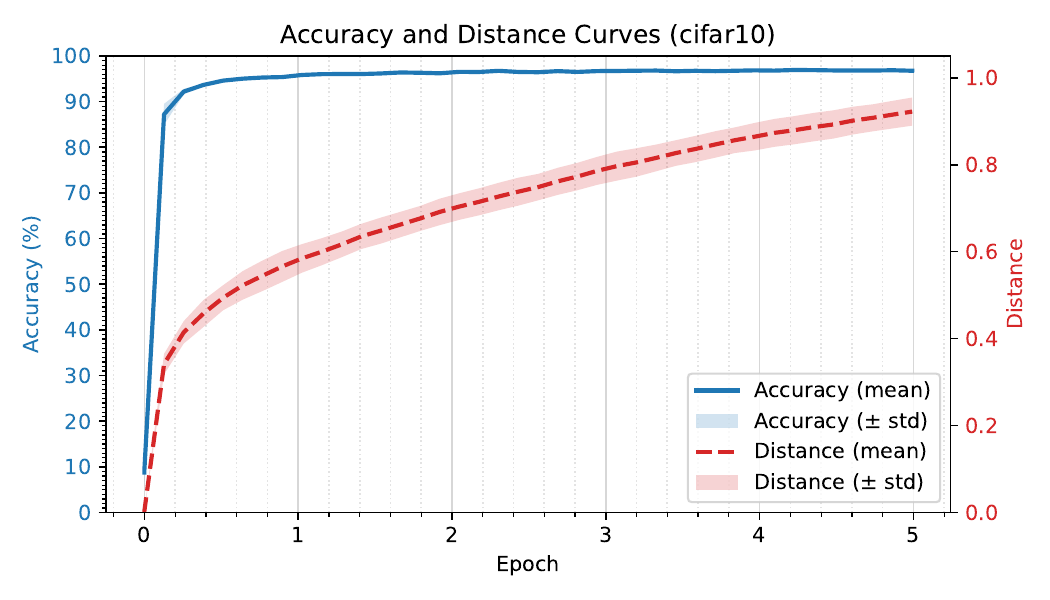}
    \caption{\textbf{CIFAR-10 (ViT-Tiny), class deletion, lr=3e-4.}
    Test accuracy (left) and parameter distance $\Delta$ (right) over training steps.}
    \label{fig:delta-traj-cifar-vit-class-lr3e-4}
\end{figure}

\begin{figure}[H]
    \centering
    \includegraphics[width=0.8\linewidth]{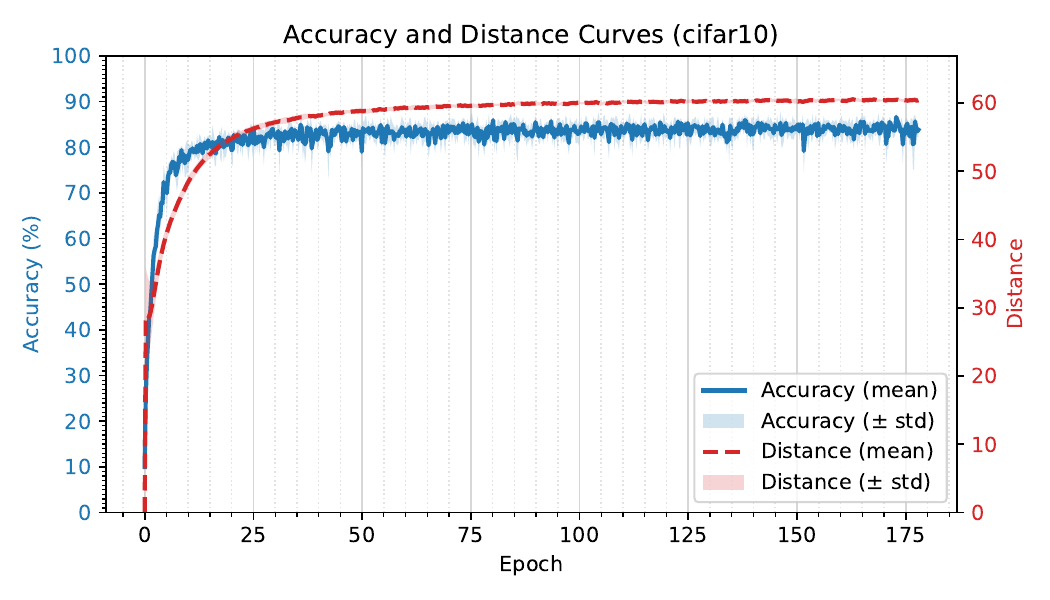}
    \caption{\textbf{CIFAR-10 (ResNet-18), random 10\% deletion, lr=0.1.}
    Test accuracy (left) and parameter distance $\Delta$ (right) over training steps.}
    \label{fig:delta-traj-cifar-rn18-10p-lr1e-1}
\end{figure}

\begin{figure}[H]
    \centering
    \includegraphics[width=0.8\linewidth]{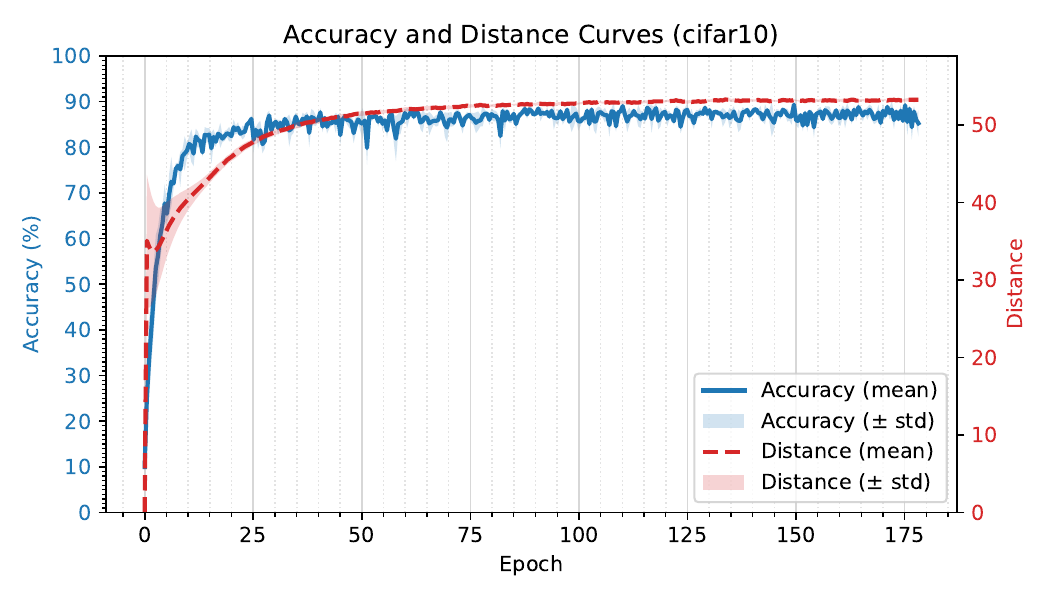}
    \caption{\textbf{CIFAR-10 (ResNet-18), class deletion, lr=0.1.}
    Test accuracy (left) and parameter distance $\Delta$ (right) over training steps.}
    \label{fig:delta-traj-cifar-rn18-class-lr1e-1}
\end{figure}

\subsection{Sensitivity to proximity calibration}
\label{app:delta-rho-sensitivity}

We additionally study how sensitive Block-wise NFT is to the instantiated value
of $\Delta(\rho)$. Recall that $\Delta(\rho)$ enters only through the certified
noise calibration: larger values correspond to a more conservative estimate of
the initial discrepancy between the full-data and retained-data training runs.

Figure~\ref{fig:delta-overestimation} shows the effect of overestimating
$\Delta(\rho)$ by a multiplicative factor $m$. Moderate overestimation has
limited effect on the final model: the final accuracy remains stable for small
increases of $\Delta(\rho)$, while larger overestimation eventually makes the
certified noisy phase too conservative and degrades utility. This supports the
view that the method does not require an exactly optimal calibration, although
overly pessimistic values of $\Delta(\rho)$ can still harm performance.

Figure~\ref{fig:budget-recalculation} illustrates the complementary
re-certification perspective. If a completed run is later evaluated under a
larger value $m\Delta(\rho)$, the same trajectory can be certified with a larger
privacy/unlearning budget as described in Appendix~\ref{app:budget-recalculation}. The increase is smooth for moderate changes in
$\Delta(\rho)$, which means that small calibration corrections do not require
rerunning unlearning; they can instead be accounted for by recalculating the
certified budget.

\begin{figure}[H]
    \centering
    \includegraphics[width=0.8\linewidth]{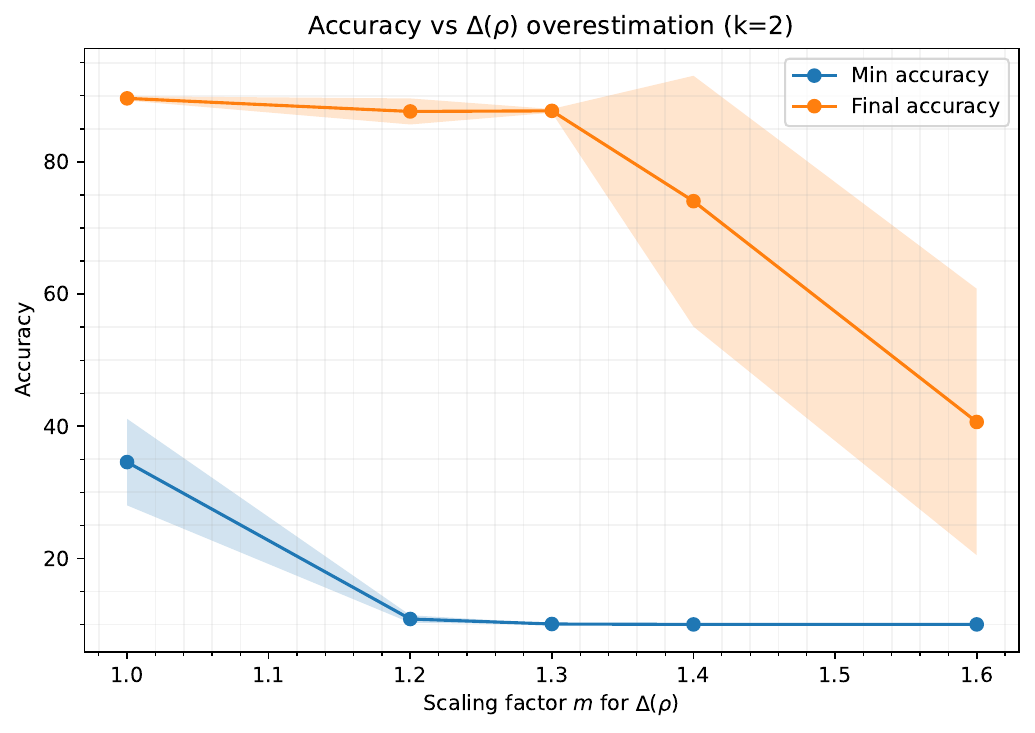}
    \caption{
    Sensitivity to $\Delta(\rho)$ overestimation ($k=2$) for ViT-Tiny/CIFAR-10. 
    The x-axis shows the factor $m$ in $\Delta(\rho) \to m\Delta(\rho)$, while the y-axis reports accuracy. 
    Here, Min accuracy denotes the accuracy immediately after the unlearning steps, before post-fine-tuning, i.e., the worst accuracy reached during the unlearning process. 
    Moderate overestimation (up to $\sim$30\%) preserves high final accuracy, while larger values lead to degradation.
    }
    \label{fig:delta-overestimation}
\end{figure}

\begin{figure}[H]
    \centering
    \includegraphics[width=1.\linewidth]{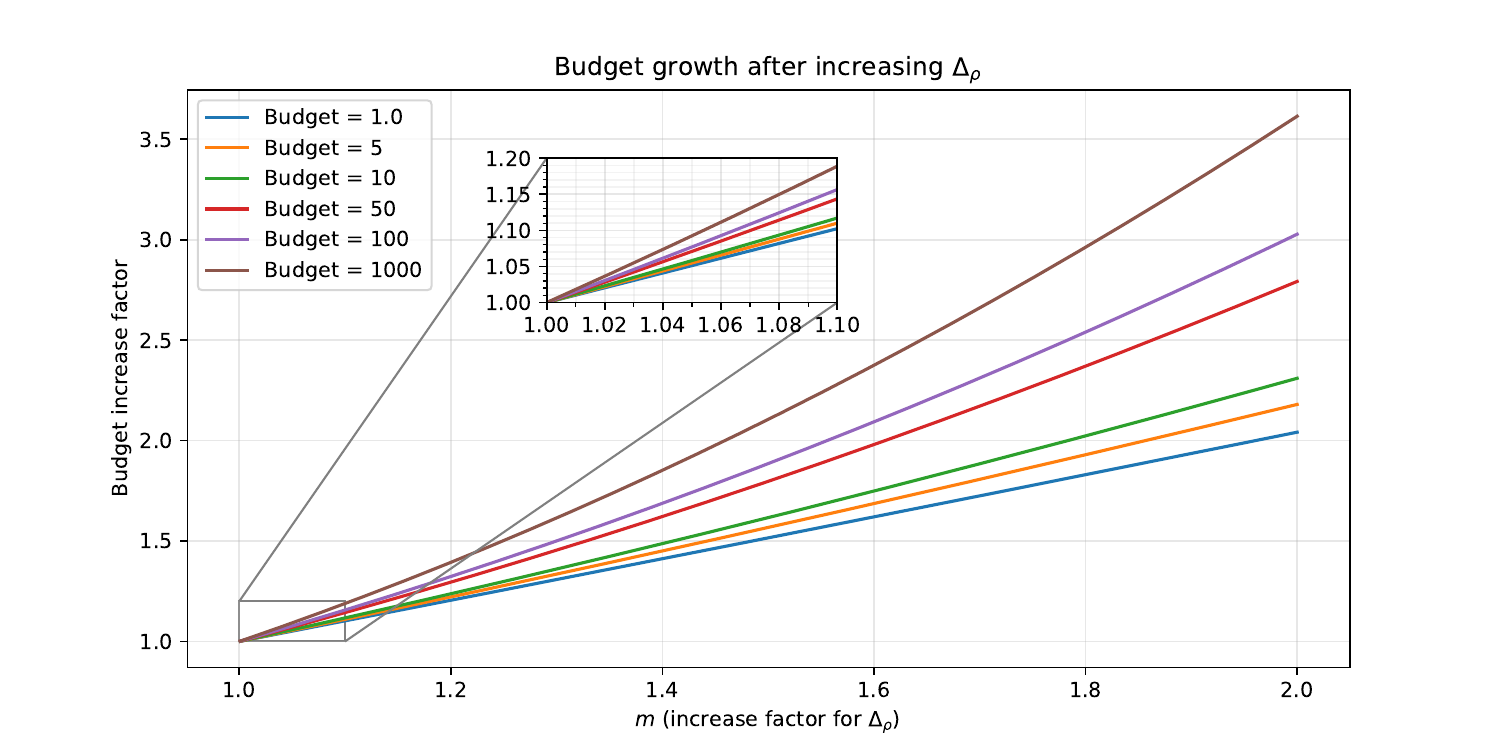}
    \caption{
    Budget recalculation under larger values of $\Delta(\rho)$. For a fixed unlearning trajectory, increasing the discrepancy estimate to
    $m\Delta(\rho)$ increases the certified budget by a smooth multiplicative factor. Thus, moderate corrections to $\Delta(\rho)$ can be handled by
    re-certification rather than rerunning unlearning.
    }
    \label{fig:budget-recalculation}
\end{figure}

\section{Broader impacts}

This work may have positive societal impact by improving practical certified unlearning methods, which can support privacy-preserving data deletion and compliance with data-removal requirements. At the same time, certified unlearning guarantees may create a false sense of complete removal if their assumptions, calibration regime, or deployment conditions are misunderstood. We therefore emphasize that our guarantees are tied to the stated certification setting and calibration procedure, and should be complemented by careful auditing in deployment.

\end{document}